\relax
\documentclass[letterpaper]{article} 
\usepackage{aaai20}  
\usepackage{times}  
\usepackage{helvet} 
\usepackage{courier}  
\usepackage[hyphens]{url}  
\usepackage{graphicx} 
\usepackage{graphicx}  

\usepackage{algorithm}
\usepackage{algorithmicx}
\usepackage{algpseudocode}
\usepackage{amsfonts}
\usepackage{amsmath}
\usepackage{amssymb}
\usepackage{mathtools}
\usepackage{amsthm}
\usepackage{booktabs}
\usepackage{subfigure}
\usepackage{float}
\usepackage{color}
\newtheorem{definition}{Definition}
\newtheorem{theorem}{Theorem}
\newtheorem{remark}{Remark}

\newtheorem{assumption}{Assumption}
\newtheorem{lemma}{Lemma}

\title{Sharper Utility Bounds for Differentially Private Models}

\author{\Large \textbf{Yilin Kang, Yong Liu, Jian Li, Weiping Wang}\\ 
}

\begin{document}

\maketitle

\begin{abstract}
In this paper, by introducing Generalized Bernstein condition, we propose the first $\mathcal{O}\big(\frac{\sqrt{p}}{n\epsilon}\big)$ high probability excess population risk bound for differentially private algorithms under the assumptions $G$-Lipschitz, $L$-smooth, and Polyak-{\L}ojasiewicz condition, based on gradient perturbation method.
If we replace the properties $G$-Lipschitz and $L$-smooth by $\alpha$-H{\"o}lder smoothness (which can be used in non-smooth setting), the high probability bound comes to $\mathcal{O}\big(n^{-\frac{\alpha}{1+2\alpha}}\big)$ w.r.t $n$, which cannot achieve $\mathcal{O}\left(1/n\right)$ when $\alpha\in(0,1]$.
To solve this problem, we propose a variant of gradient perturbation method, \textbf{max$\{1,g\}$-Normalized Gradient Perturbation} (m-NGP).
We further show that by normalization, the high probability excess population risk bound under assumptions $\alpha$-H{\"o}lder smooth and Polyak-{\L}ojasiewicz condition can achieve $\mathcal{O}\big(\frac{\sqrt{p}}{n\epsilon}\big)$, which is the first $\mathcal{O}\left(1/n\right)$ high probability excess population risk bound w.r.t $n$ for differentially private algorithms under non-smooth conditions.
Moreover, we evaluate the performance of the new proposed algorithm m-NGP, the experimental results show that m-NGP improves the performance of the differentially private model over real datasets.
It demonstrates that m-NGP improves the utility bound and the accuracy of the DP model on real datasets simultaneously.
\end{abstract}

\section{1. Introduction}

Machine learning has been widely used and found effective in many fields in recent years \cite{singha2021role,swapna2021diabetes,ponnusamy2021smart}.
When training machine learning models, tremendous data was collected, and the data often contains sensitive information of individuals, which may leakage personal privacy \cite{reza2017membership,nicholas2019thesecret}.
Under these circumstances, the privacy of machine learning models is of great importance.

Differential Privacy (DP) \cite{dwork2006calibrating,dwork2014algorithmic} is a theoretically rigorous tool to prevent sensitive information leakage.
It introduces random noise when training and blocks adversaries from inferring any single individual included in the dataset by observing the model.
The mathematical definition of DP is well accepted and relative technologies are performed by Google \cite{erlingsson2014rappor}, Apple \cite{mcmillan2016apple} and Microsoft \cite{ding2017collecting}.
As such, DP has attracted attentions from researchers and has been applied to numerous machine learning problems \cite{ullman2019efficiently,xu2019ganobfuscator,bernstein2019differentially,wang2019principal,heikkila2019differentially,kulkarni2021differentially,bun2021differentially,nguyen2021differentially}.

There are mainly three approaches to guarantee differential privacy: output perturbation \cite{chaudhuri2011differentially}, objective perturbation \cite{chaudhuri2011differentially}, and gradient perturbation \cite{song2013stochastic}.
Considering that gradient descent is a widely used optimization method, the gradient perturbation method can be used for a wide range of applications.
Besides, adding random noise to the gradient allows the model to escape local minima \cite{raginsky2017nonconvex}, so we focus on the gradient perturbation method to guarantee differential privacy in this paper.

In this paper, we aim to minimize the population risk, and measure the utility of the differentially private model by the excess population risk.
To get the excess population risk, an important step is to analyze the generalization error (the reason is demonstrated in Section 3).
Complexity theory \cite{bartlett2002localized} and algorithm stability theory \cite{bousquet2002stability} are popular tools to analyze the generalization error.
On one hand, \cite{chaudhuri2011differentially} applied the complexity theory and achieved an $\mathcal{O}\big(\max\{\frac{1}{\sqrt{n}},\sqrt[2/3]{\frac{p}{n\epsilon}}\}\big)$ high probability excess population risk bound under the assumption of strongly convex;
\cite{kifer2012private} achieved $\mathcal{O}\big(\frac{\sqrt{p}}{n\epsilon}\big)$ expected excess population risk bound via complexity theory.
On the other hand, the sharpest known high probability generalization bounds for DP algorithms analyzed via stability theory under different assumptions  \cite{wu2017bolt,bassily2019private,feldman2020private,bassily2020stability,wang2021differentially} are $\mathcal{O}\big(\frac{\sqrt{p}}{n\epsilon}+\frac{1}{\sqrt{n}}\big)$ or $\mathcal{O}\big(\frac{\sqrt[4]{p}}{\sqrt{n}\epsilon}\big)$, containing an inevitable $\mathcal{O}\big(\frac{1}{\sqrt{n}}\big)$ term, which is a bottleneck on the utility analysis.
Thus, we are focusing on the following question, which is still an open problem:

\textit{Can we achieve the high probability excess risk bounds with rate  $\mathcal{O}(\frac{\sqrt{p}}{n\epsilon})$ for DP models via uniform stability?}

By introducing \textit{Generalized Bernstein condition} \cite{koltchinskii2006local}, this paper answers the question positively.
We remove the $\mathcal{O}\big(\frac{1}{\sqrt{n}}\big)$ term in the generalization error and provide the first high probability excess population risk bound with order $\mathcal{O}\big(\frac{\sqrt{p}}{n\epsilon}\big)$ in the setting of DP.
Comparing with previous high probability bounds, the improvement is approximately up to $\mathcal{O}\left(\sqrt{n}\right)$.
The contributions of this paper include:
(1) We prove that by introducing Generalized Bernstein condition \cite{koltchinskii2006local}, the high probability excess population risk can be improved to $\mathcal{O}\big(\frac{\sqrt{p}}{n\epsilon}\big)$, under Lipschitz and smooth assumptions.
To our knowledge, this is the first $\mathcal{O}\big(\frac{\sqrt{p}}{n\epsilon}\big)$ high probability excess population risk bound for DP model.
(2) We relax the assumptions $G$-Lipschitz and $L$-smooth, by introducing $\alpha$-H{\"o}lder smooth.
Under these assumptions, we prove that the high probability excess population risk bound comes to $\mathcal{O}\big(\frac{\sqrt[4]{p}}{\sqrt{\epsilon}}n^\frac{-\alpha}{1+2\alpha}\big)$.
Considering that $\alpha\in(0,1]$, the result cannot achieve $\mathcal{O}\big(\frac{1}{n}\big)$ w.r.t $n$, but better than previous one w.r.t $p$ and $\epsilon$.
(3) To overcome the bottleneck, we design a variant of gradient perturbation method, called \textbf{max $\{1,g\}$-Normalized Gradient Perturbation} (m-NGP) algorithm.
Via this new proposed algorithm, we prove that under the assumptions $\alpha$-H{\"o}lder smooth and PL condition, the high probability excess population risk bound can be improved to $\mathcal{O}\big(\frac{\sqrt{p}}{n\epsilon}\big)$.
To the best of our knowledge, this is the first $\mathcal{O}\big(\frac{\sqrt{p}}{n\epsilon}\big)$ high probability excess population risk bound for non-smooth loss in the field of differential privacy.
(4) To evaluate the performance of our proposed max $\{1,g\}$-Normalized Gradient Perturbation algorithm, we perform experiments on several real datasets.
The experimental results show that m-NGP improves the accuracy and the convergence rate of the differentially private model on real datasets.

The rest of the paper is organized as follows.
Related work is given in Section 2.
Preliminaries are introduced in Section 3.
In Section 4, we propose sharper utility bounds under different assumptions and design a variant of gradient perturbation method, \textbf{max $\{1,g\}$-Normalized Gradient Perturbation}.
The experimental results are shown in Section 5.
Finally, we conclude the paper in Section 6.

\begin{table*}[t]
	\caption{Previous excess population risk bounds and ours under different assumptions}
	\label{tab1}
	\begin{center}
		\begin{small}
			\begin{sc}
				\begin{tabular}{ccccc}
					\toprule
					& Assumptions  & Method & Utility Bound & Type \\
					\midrule
					\cite{bassily2019private} & Lipschitz, smooth, convex & Gradient & $\mathcal{O}\left(\frac{\sqrt{p}}{n\epsilon}+\frac{1}{\sqrt{n}}\right)$ & E. \\
					\hline \\
					\cite{feldman2020private} & Lipschitz, convex & Gradient & $\mathcal{O}\left(\frac{\sqrt{p}}{n\epsilon}+\frac{1}{\sqrt{n}}\right)$ & E. \\
					\hline \\
					\cite{feldman2020private} & Lipschitz, strongly convex & Gradient & $\mathcal{O}\left(\frac{d}{n^2\epsilon^2}+\frac{1}{n}\right)$ & E. \\
					\hline \\
					\cite{bassily2020stability} & Lipschitz, convex & Gradient & $\mathcal{O}\left(\frac{\sqrt{p}}{n\epsilon}+\frac{1}{\sqrt{n}}\right)$ & H.P.\\
					\hline \\
					\cite{wang2021differentially} & $\alpha$-H{\"o}lder smooth, convex & Gradient & $\mathcal{O}\left(\frac{\sqrt{p}}{n\epsilon}+\frac{1}{\sqrt{n}}\right)$ & H.P. \\
					\hline \\
					\cite{wang2021differentially} & $\alpha$-H{\"o}lder smooth, convex & Output & $\mathcal{O}\left(\frac{\sqrt[4]{p}}{\sqrt{n}\epsilon}\right)$ & H.P. \\
					\hline \\
					Ours & Lipschitz, smooth, PL condition & Gradient & $\mathcal{O}\left(\frac{\sqrt{p}}{n\epsilon}\right)$ & H.P. \\
					\hline \\
					Ours & $\alpha$-H{\"o}lder smooth, PL condition & Gradient & $\mathcal{O}\left(\frac{\sqrt[4]{p}}{n^{\frac{\alpha}{1+2\alpha}}\epsilon^{\frac{1}{2}}}\right)$ & H.P. \\
					\hline
					Ours (m-NGP) & $\alpha$-H{\"o}lder smooth, PL condition & Gradient & $\mathcal{O}\left(\frac{\sqrt{p}}{n\epsilon}\right)$ & H.P. \\
					\bottomrule
				\end{tabular}
			\end{sc}
		\end{small}
	\end{center}
\end{table*}

\section{2. Related Work}

\cite{dwork2006calibrating} proposed the mathematical definition of differential privacy for the first time.
Then, it was developed to protect the privacy in the field of machine learning (e.g. Empirical Risk Minimization (ERM)) via output perturbation, objective perturbation, and gradient perturbation methods.
For DP-ERM formulations, \cite{chaudhuri2011differentially} first proposed output perturbation and objective perturbation methods, and \cite{song2013stochastic} first proposed the gradient perturbation method.
Based on these works, \cite{kifer2012private,bassily2014private,abadi2016deep,wang2017differentially,zhang2017efficient,wu2017bolt,bassily2019private,feldman2020private,bassily2020stability} further improved the results under different assumptions.

Among the works mentioned above, some of them analyzed the privacy guarantees \cite{song2013stochastic,abadi2016deep}, some of them discussed the excess empirical risk bound \cite{wang2017differentially,zhang2017efficient,wu2017bolt}.
Some works discussed the excess population risk under expectation, from different points of view, such as complexity theory, optimization theory, and stability theory: \cite{kifer2012private} achieved an $\mathcal{O}\big(\frac{\sqrt{p}}{n\epsilon}\big)$ excess population risk bound via complexity theory under expectation condition; \cite{bassily2014private} achieved similar expected excess population risk bound under convexity assumption, via optimization theory; \cite{wang2019differentially} proposed an $\mathcal{O}\big(\frac{p}{\log(n)\epsilon^2}\big)$ excess population risk bound under non-convex condition in expectation, via Langevin Dynamics method \cite{gelfand1991recursive} and the stability of Gibbs algorithm; and \cite{feldman2020private} gives expected population risk bound of the order $\mathcal{O}\big(\frac{1}{n}+\frac{p}{n^2}\big)$, under strongly convex condition.

However, the behavior of the algorithm within a single or few runs cannot be well captured by expectation bounds, which is related to the probabilistic nature.
In addition, in practical applications such as deep learning, it is often the case that the algorithm runs only once since the training process may take a long time.
Therefore, obtaining a high probability bound is essential to ensure the performance of the algorithm on a single or few runs.
So we focus on the high probability bound in this paper.
Meanwhile, we concentrate on stability theory.
Among many notions of stability, uniform stability is arguably the most popular one, which yields exponential generalization bounds.
Via uniform stability, the high probability excess population risk bounds under different assumptions given by previous works all contain an $\mathcal{O}\big(\frac{1}{\sqrt{n}}\big)$ term, details can be found in Table \ref{tab1}.
The reason is that when analyzing the generalization error, the technical routes follow works \cite{bousquet2002stability,hardt2016train}.
Besides, when analyzing the stability, previous works always do not consider the injected noise (e.g. \cite{wang2021differentially}), or assume the random noise injected into adjacent datasets is the same.
However, this is not reasonable, because `adjacent dataset' is also the basis of DP.
With the same random noise, it is hard to say that `DP is guaranteed'.

In this paper, we consider the random noise injected into adjacent datasets and analyze the stability under the noisy version.
By introducing the \textit{Generalized Bernstein condition} \cite{koltchinskii2006local}, we remove the $\mathcal{O}\big(\frac{1}{\sqrt{n}}\big)$ term when combining the stability and the generalization error, and further improve the excess population risk bound of differentially private models.
The improved convergence rate is up to $\mathcal{O}\big(\frac{\sqrt{p}}{n\epsilon}\big)$, which positively answers the question given in Section 1: Can the high probability excess population risk bound achieve $\mathcal{O}\left(1/n\right)$ w.r.t $n$.
The improvements are shown in Table \ref{tab1}.
For `TYPE', E. means expectation bound and H.P. means the high probability bound.

Table \ref{tab1} first shows that by adding more assumptions, we achieve a better high probability excess population risk bound, $\mathcal{O}\big(\frac{\sqrt{p}}{n\epsilon}\big)$, which is state-of-the-art to the best of our knowledge.
Then, we relax the assumptions and achieve $\mathcal{O}\big(\frac{\sqrt[4]{p}}{n^{\frac{\alpha}{1+2\alpha}}\epsilon^{\frac{1}{2}}}\big)$ high probability bound, but it cannot achieve the same bound ($\mathcal{O}\left(1/n\right)$ w.r.t $n$) under the condition that the loss function is Lipschitz, smooth, and satisfies PL condition.
To overcome this problem, we propose an algorithm called m-NGP, and achieve the $\mathcal{O}\big(\frac{\sqrt{p}}{n\epsilon}\big)$ result under the same assumptions: $\alpha$-H{\"o}lder smooth and PL condition.

Moreover, although it is hard to directly compare the PL condition with convexity, PL condition can be applied to many non-convex conditions (more information can be found in Section 4.2).
Besides, PL condition is weaker compared with strongly convex condition, and one of the best population risks under strongly convex condition is $\mathcal{O}\big(\frac{1}{n}+\frac{p}{n^2}\big)$ \cite{feldman2020private} (line 2 in Table \ref{tab1}).
However, the result is an expectation one, different from ours.
In this paper, we analyze the excess population risk bound of DP algorithm under high probability and PL cases, different from previous scenarios.

\section{3. Preliminaries}

In this paper, we assume that there are $n$ data instances in dataset $D$, i.e. $D=\{z_1,\cdots,z_n\}$ where $z=(x,y)$ with input $x\in\mathcal{X}$ and label $y\in\mathcal{Y}$, and $\mathcal{Z}=\mathcal{X}\times\mathcal{Y}$.
The data space is denoted by $\mathcal{D}$ and the parameter space is denoted by $\mathcal{C}$, the loss function $\ell$ is defined as $\ell(\cdot,\cdot):\mathcal{D}\times\mathcal{C}\rightarrow\mathbb{R}$.
Databases $D,D'\in\mathcal{D}^n$ differing by one data instance are denoted as $D\sim D'$, called \textit{adjacent databases}.
For a given vector $\bold{x}=[x_1,\cdots,x_d]^T$, its $\ell_p$-norm is $\|\bold{x}\|_p=(\sum\nolimits_{i=1}^{d}|x_i|^p)^\frac{1}{p}$.
And $A\lesssim B$ represents that there exists some constant $c>0$, $A\leq cB$.

\begin{definition}[Differential Privacy \cite{dwork2006calibrating}]\label{DP}
	A randomized algorithm: $\mathcal{A}:\mathcal{D}^n\rightarrow\mathbb{R}^p$ is ($\epsilon,\delta$)-differential privacy (DP) if for all $D\sim D'$ and events $S\in range(\mathcal{A})$:
	\begin{equation*}
	\mathbb{P}\left[\mathcal{A}(D)\in S\right]\leq e^\epsilon\mathbb{P}\left[\mathcal{A}(D')\in S\right]+\delta.
	\end{equation*}
\end{definition}

Definition \ref{DP} implies that the adversaries cannot infer whether an individual participates when training the machine learning model, because essentially the same distributions will be drawn over any adjacent datasets.
Some kind of attacks, such as membership inference attack, attribute inference attack, and memorization attack, can be thwarted by DP \cite{backes2016membership,bargav2019evaluating,nicholas2019thesecret}.

Throughout this paper, we focus on gradient perturbation method to guarantee ($\epsilon,\delta$)-DP, the paradigm is based on gradient descent: at iteration $t$,
\begin{equation}\label{DPGD}
\hat{\theta}_t\leftarrow\hat{\theta}_{t-1}-\eta_t\left(\nabla_\theta R_n(\hat{\theta}_{t-1})+b_t\right),
\end{equation}
where $\eta_t$ is the learning rate at iteration $t$, $b_t$ is the random noise injected into the gradient, $\hat{\theta}$ is corresponding model with privacy, and $R_n(\theta)$ is the empirical risk, defined as $R_n(\theta)\coloneqq\frac{1}{n}\sum_{i=1}^{n}\ell(z_i,\theta)$.

In this paper, we focus on minimizing the population risk $R(\theta)=\mathbb{E}_{z\sim \mathcal{D}}\left[\ell(z,\theta)\right]$.
In the setting of differential privacy, the excess population risk is defined by $R(\hat{\theta})-\min_{\theta\in\mathcal{C}}R(\theta)$, which can be decomposed into:
\begin{equation}\label{DeEPR}
\begin{aligned}
&R(\hat{\theta}_n)-\min_{\theta\in\mathcal{C}}R(\theta) \\
&=\underbrace{R(\hat{\theta}_n)-R_n(\hat{\theta}_n)}_{A}+\underbrace{R_n(\hat{\theta}_n)-R_n(\theta^*)}_{B} \\
&\quad+R_n(\theta^*)-R(\theta^*),
\end{aligned}
\end{equation}
where $\theta^*=\arg\min_{\theta\in\mathcal{C}}R(\theta),\theta_n^*=\arg\min_{\theta\in\mathcal{C}}R_n(\theta)$.
In (\ref{DeEPR}), part $A$ is exactly the generalization.
Via the definition of $\theta_n^*$, we have $R_n(\hat{\theta}_n)-R_n(\theta^*)\leq R_n(\hat{\theta}_n)-R_n(\theta_n^*)$, which bounds part $B$ by the optimization error (also called the excess empirical risk).
In this way, we answer the question mentioned in Section 1: Why generalization error is an important step towards the excess population risk.

To get the generalization error, algorithm stability theory is a popular tool, we introduce uniform stability and uniform argument stability here.
\begin{definition}[Uniform Stability \cite{bousquet2002stability}]\label{US}
	An algorithm $\theta_n$ is $\gamma$-uniformly stable if for any $S=\{z_1,\cdots,z_{i},\cdots,z_n\}$ and $S'=\{z_1,\cdots,z_i',\cdots,z_n\}$, where $i=1,\cdots,n$, it holds that
	\begin{equation*}
	\left|\ell(z,\theta_n(S))-\ell(z,\theta_n(S'))\right|\leq\gamma.
	\end{equation*}
\end{definition}
In this paper, we use notation $\theta_n(S)$ for both algorithm and model parameter.
By Definition \ref{US}, it is easy to follow that the uniform stability measures the upper bound of the difference (on the loss function) between the models derived from adjacent datasets.

\begin{definition}[Uniform Argument Stability \cite{bassily2020stability}]\label{UAS}
	Algorithm $\theta_n$ is $\gamma$-uniformly argument stable if for any $S=\{z_1,\cdots,z_{i},\cdots,z_n\}$ and $S'=\{z_1,\cdots,z_i',\cdots,z_n\}$, where $i=1,\cdots,n$, it holds that
	\begin{equation*}
	\left\Vert\theta_n(S)-\theta_n(S')\right\Vert_2\leq\gamma.
	\end{equation*}
\end{definition}

Definition \ref{UAS} shows that the uniform argument stability measures the upper bound of the difference (on the model parameter) between the models derived from adjacent datasets.

Furthermore, we introduce some assumptions.

\begin{assumption}[$G$-Lipschitz]\label{a1}
	The loss function $\ell:\mathcal{D}\times\mathcal{C}\rightarrow\mathbb{R}$ is $G$-Lipschitz over $\theta$ if for any $z\in\mathcal{D}$ and $\theta_1,\theta_2\in\mathcal{C}$, we have:
	$|\ell(z,\theta_1)-\ell(z,\theta_2)|\leq G\|\theta_1-\theta_2\|_2$.
\end{assumption}

With Assumption \ref{a1}, one can easily get that if the loss function is $G$-Lipchitz, then $\gamma$-uniformly argument stability implies $G\gamma$-uniformly stability.

\begin{assumption}[$L$-smooth]\label{a2}
	The loss function $\ell:\mathcal{D}\times\mathcal{C}\rightarrow\mathbb{R}$ is $L$-smooth over $\theta$ if for any $z\in\mathcal{D}$ and $\theta_1,\theta_2\in\mathcal{C}$, we have:
	$\|\nabla_\theta\ell(z,\theta_1)-\nabla_\theta\ell(z,\theta_2)\|_2\leq L\|\theta_1-\theta_2\|_2$.
\end{assumption}
If loss function $\ell(\cdot,\cdot)$ is differentiable, smoothness yields: $\ell(z,\theta_1)-\ell(z,\theta_2)\leq\langle\nabla_\theta\ell(z,\theta_2),\theta_1-\theta_2\rangle+\frac{L}{2}\left\Vert\theta_1-\theta_2\right\Vert_2^2$.

Assumptions $G$-Lipschitz and $L$-smooth are commonly used in the utility analysis of DP machine learning \cite{chaudhuri2011differentially,kifer2012private,abadi2016deep,bassily2019private,feldman2020private,bassily2020stability}.
To relax the Lipschitz and smoothness assumptions, we introduce the $\alpha$-H{\"o}lder smoothness of the loss function:
\begin{assumption}[$\alpha$-H{\"o}lder smooth]\label{a3}
	Let $\alpha\in(0,1]$.
	The loss function $\ell:\mathcal{D}\times\mathcal{C}\rightarrow\mathbb{R}$ is $\alpha$-H{\"o}lder smooth over $\theta$ with parameter $H$ if for any $z\in\mathcal{D}$ and $\theta_1,\theta_2\in\mathcal{C}$, we have:
	$\|\nabla_\theta\ell(z,\theta_1)-\nabla_\theta\ell(z,\theta_2)\|_2\leq H\|\theta_1-\theta_2\|_2^\alpha$.
\end{assumption}
\begin{lemma}[\cite{ying2017unregularized}]\label{t1}
	If the loss function $\ell(\cdot,\cdot)$ is differentiable, then Assumption \ref{a3} yields  $\ell(z,\theta_1)-\ell(z,\theta_2)\leq\langle\nabla_\theta\ell(z,\theta_2),\theta_1-\theta_2\rangle+\frac{H}{\alpha+1}\left\Vert\theta_1-\theta_2\right\Vert_2^{\alpha+1}$.
\end{lemma}

By the definition, it is easy to follow that if $\alpha=1$, it is equivalent to $H$-smooth; and if $\alpha\rightarrow0$, it satisfies the Lipschitz property given in Assumption \ref{a1}.
Besides, with bounded parameter space, i.e. $\|\mathcal{C}\|_2\leq M_\mathcal{C}$, $\alpha$-H{\"o}lder smoothness immediately implies $\max\{2HM_\mathcal{C},H\}$-Lipschitz.
Moreover, Assumption \ref{a3} instantiates many non-smooth loss functions.
For example, the $q$-norm hinge loss $\ell(z,\theta)=\left(\max\left(0,1-y\langle\theta,z\rangle\right)\right)^q$ for classification and the $q$-th power absolute distance loss $\ell(z,\theta)=|y-\langle\theta,z\rangle|^q$ for regression \cite{lei2020fine}, whose $\ell$ are ($q-1$)-H{\"o}lder smooth if $q\in(1,2]$ \cite{li2021improved}.
Lemma \ref{t1} shows that H{\"o}lder smoothness shares similar property with smoothness defined in Assumption \ref{a2}.

\section{4. Sharper Utility Bounds for Differentially Private Models}

\subsection{4.1. Privacy Guarantees}
Before analyzing the excess population risk bound, we first discuss the privacy guarantees in this section.
\cite{abadi2016deep} proposed the moments accountant method to measure the privacy costs of DP model training by stochastic gradient descent (SGD), \cite{wang2017differentially} further analyzed it under the setting of gradient descent (GD).
In this paper, we focus more on the utility analysis, to improve the excess population risk, so we directly apply it to the gradient perturbation method.

\begin{lemma}[\cite{wang2017differentially}]\label{l1}
	In gradient perturbation method in (\ref{DPGD}), if Assumption \ref{a1} holds, then for $\epsilon,\delta>0$, it is ($\epsilon,\delta$)-DP if the Gaussian random noise $b_t\sim\mathcal{N}(0,\sigma^2I_p)$, and for some constant $c$,
	\begin{equation*}
	\sigma^2=c\frac{G^2T\log(1/\delta)}{n^2\epsilon^2}.
	\end{equation*}
\end{lemma}

\begin{remark}\label{r1}
	Lemma \ref{l1} only assumes the loss function to be $G$-Lipschitz.
	If we assume that $\ell(\cdot,\cdot)$ is $\alpha$-H{\"o}lder smooth with parameter $H$, then $G$ can be replaced by $\max\{2HM_\mathcal{C},H\}$ as discussed above.
\end{remark}

\subsection{4.2. Analysis of the excess population risk}

Before analyzing the excess population risk, we first introduce some assumptions.

Most of the previous works assumed that the loss function is convex (or strongly convex) when analyzing the empirical and population risks.
In this paper, we use the Polyak-{\L}ojasiewicz (PL) condition to replace convexity.
\begin{assumption}[Polyak-{\L}ojasiewicz condition]\label{a5}
	Function $f(\theta)$ satisfies the Polyak-{\L}ojasiewicz (PL) condition if there exists $\mu>0$ and for every $\theta$,
	\begin{equation*}
	\left\Vert\nabla_\theta f(\theta)\right\Vert_2^2\geq2\mu\left(f(\theta)-f(\theta^*)\right),
	\end{equation*}
	where $\theta^*=\arg\min_{\theta\in\mathcal{C}}f(\theta)$.
	
	In this paper, we assume the empirical risk and the population risk both satisfy the PL condition\footnote{PL condition can be directly derived from strongly convex \cite{karimi2016linear}, so all the results given in this paper hold when it comes to the strongly convex conditions.}.
\end{assumption}

The Polyak-{\L}ojasiewicz condition is one of the weakest curvature conditions \cite{karimi2016linear,li2021improved}, weaker than `one-point convexity' \cite{kleinberg2018an}, `star convexity' \cite{zhou2019sgd}, and `quasar convexity' \cite{hinder2020near}.
It is widely used in the analysis of non-convex learning \cite{wang2017differentially,charles2018stability,lei2020sharper,lei2021learning} and many popular non-convex objective functions satisfy the PL condition, such as: matrix
factorization \cite{liu2016quadratic}, robust regression \cite{liu2016quadratic}, neural networks with one hidden layer \cite{li2017convergence}, mixture of two Gaussians \cite{balakrishnan2017statistical}, ResNets with linear
activations \cite{hardt2017identity}, linear dynamical systems \cite{hardt2018gradient}, phase retrieval \cite{sun2018a}, and blind deconvolution \cite{li2019rapid}.

\begin{remark}\label{r8}
	\cite{karimi2016linear} shows that the Polyak-{\L}ojasiewicz inequality directly implies the following inequality: $R_n(\theta)-R_n(\theta_n^*)\geq\frac{\mu}{2}\|\theta-\theta_n^*\|_2^2$, which is called the Quadratic Growth (QG) condition.
	In the following, we use QG to bound the argument stability of the gradient perturbation algorithm.
\end{remark}

Then, with Assmption \ref{a5}, we discuss the uniform argument stability of the private model.

\begin{lemma}\label{l9}
	In gradient perturbation method (\ref{DPGD}), if Assumptions \ref{a1}, \ref{a5} hold, then
	\begin{equation*}
	\left\Vert\hat{\theta}_n-\hat{\theta}_n'\right\Vert_2\leq2\sqrt{2}\sqrt{\frac{R_n(\hat{\theta}_n)-R_n(\theta_n^*)}{\mu}}+\frac{4G}{\mu n},
	\end{equation*}
	where $\hat{\theta}_n$ and $\hat{\theta}_n'$ denote the models derived from adjacent datasets $S$ and $S'$.
\end{lemma}

The proof is given in Appendix A.1.
By the QG condition (implied by PL inequality) discussed in Remark \ref{r8}, Lemma \ref{l9} connects the uniform argument stability with the empirical risk.
And as a result, we only need to consider the noise when analyzing the empirical risk.
Besides, when it comes to the stability of DP models, previous results often assume that the noise added to the gradient in each iteration is the same for adjacent datasets $S$ and $S'$ (e.g. \cite{wang2021differentially}).
This is not reasonable because noise injection is an independent process, so we expand it in this paper.

And to remove the $\mathcal{O}(1/\sqrt{n})$ term in previous results, we further need the Generalized Bernstein condition when analyzing the excess population risk.
\begin{assumption}[Generalized Bernstein condition \cite{koltchinskii2006local}]\label{a4}
	We say the loss function $\ell$ satisfies the generalized Bernstein condition if for some $B>0$ for any $\theta\in\mathcal{C}$, we have:
	\begin{equation*}
	\mathbb{E}\left[\left(\ell(z,\theta)-\ell(z,\theta^*)\right)^2\right]\leq B\left(R(\theta)-R(\theta^*)\right).
	\end{equation*}
\end{assumption}

\begin{remark}\label{r7}
	Here, we discuss the connections between Assumptions \ref{a1}, \ref{a5} and \ref{a4}.
	Via Assumption \ref{a1}, we have $\mathbb{E}\left[\left(\ell(z,\theta)-\ell(z,\theta^*)\right)^2\right]\leq G^2\|\theta-\theta^*\|_2^2$.
	And Remark \ref{r8} shows that if $R(\theta)$ satisfies the Polyak-{\L}ojasiewicz condition, we have $\frac{\mu}{2}\|\theta-\theta^*\|_2^2\leq R(\theta)-R(\theta^*)$.
	Combining these inequalities together, we observe that if Assumptions \ref{a1} and \ref{a5} hold, then the loss function $\ell(\cdot)$ satisfies Assumption \ref{a4} with parameter $B=\frac{2G^2}{\mu}$.
\end{remark}

With the stability of the private model and Assumption \ref{a4}, now we come to the excess population risk.

\begin{theorem}\label{t2}
	If Assumptions \ref{a1}, \ref{a2}, and \ref{a5} hold, the loss function is bounded, i.e. $0\leq\ell(\cdot,\cdot)\leq M_\ell$, taking $\sigma$ given by Lemma \ref{l1}, $T=\mathcal{O}\left(\log(n)\right)$, $\eta_1=\cdots=\eta_T=\frac{1}{L}$, if $\zeta\in(\exp(-p/8),1)$, then with probability at least $1-\zeta$:
	\begin{equation*}
	\begin{aligned}
	&R(\hat{\theta}_n)-R(\theta^*) \\
	&\leq c_1\frac{G\log^{1.5}(n)\sqrt{p\log(1/\delta)}}{n\epsilon}\left(1+\left(\frac{8\log(T/\zeta)}{p}\right)^{1/4}\right) \\
	&\quad+c_2\frac{G^2p\log(n)\log(1/\delta)}{n^2\epsilon^2}\left(1+\left(\frac{8\log(T/\zeta)}{p}\right)^{1/4}\right)^2 \\
	&\quad+c_3\frac{\log(n)}{n}.
	\end{aligned}
	\end{equation*}
	for some constants $c_1,c_2,c_3>0$.
\end{theorem}

Detailed proof can be found in Appendix A.3, we give a proof sketch here.
First, by Lemma \ref{l9} and Lipschitzness, we get the uniform stability of gradient perturbation (\ref{DPGD}).
Then, we analyze the generalization error via stability theory.
By novel decomposition method, via Assumption \ref{a4} and its moments bound, we couple term $R(\hat{\theta}_n)-R_n(\hat{\theta}_n)$ and term $R_n(\theta^*)-R(\theta^*)$ in (\ref{DeEPR}) together, to remove the $\mathcal{O}\left(1/\sqrt{n}\right)$ term in the generalization error.
In this way, a better excess population risk bound is achieved.
The proof is motivated by \cite{klochkov2021stability} in the non-private case.
The key difference is that in the setting of DP, random noise is injected into the algorithm.
In \cite{klochkov2021stability}, a key step to analyze the generalization error is summing $X_i=\mathbb{E}'\left[\ell(z_i,\theta_n')-\ell(z_i,\theta^*)\right]$ for $i=1,\cdots,n$, where $\theta_n'$ is derived from an independent copy of the original dataset and $\mathbb{E}'$ means the expectation over the independent copy.
When summing, $X_i$ is required to be zero mean.
However, in the cases of DP, if we replace $\theta_n'$ by $\hat{\theta}_n'$, then $X_i$ are not zero mean.
Besides, for output perturbation, a common way to decompose the excess population risk is $R(\hat{\theta}_n)-R(\theta^*)\leq R(\hat{\theta}_n)-R(\theta_n)+R(\theta_n)-R_n(\theta_n)+R_n(\theta_n)-R_n(\theta_n^*)+R_n(\theta^*)-R(\theta^*)$, which naturally solves the problem mentioned above (the generalization error is discussed over the non-private model).
However, when it comes to the gradient perturbation method, we cannot solve the problem in this way, because the random noise is coupled with the gradient.
So, we decouple the noise terms and overcome the challenge by the moment Bernstein inequality.

\begin{remark}\label{r6}
	If omitting $\log(\cdot)$ and constant terms, the excess population risk bound comes to:
	\begin{equation*}
	\mathcal{O}\left(\frac{p}{n^2\epsilon^2}+\frac{\sqrt{p}}{n\epsilon}+\frac{1}{n}\right)=\mathcal{O}\left(\frac{\sqrt{p}}{n\epsilon}\right).
	\end{equation*}
	The result is better than previous ones containing $\mathcal{O}(\frac{1}{\sqrt{n}})$ terms.
	And it is the first high probability excess population risk bound over DP algorithm overcoming the $\mathcal{O}(n^{-1/2})$ bottleneck, to the best of our knowledge.
\end{remark}


Then, we replace Assumptions \ref{a1} ($G$-Lipschitz) and \ref{a2} ($L$-smooth) by Assumption \ref{a3} ($\alpha$-H{\"o}lder smooth).

\begin{theorem}\label{t3}
	If Assumptions \ref{a3} and \ref{a5} hold, the loss function and the parameter space are bounded, i.e. $0\leq\ell(\cdot,\cdot)\leq M_\ell$, $\|\mathcal{C}\|_2\leq M_\mathcal{C}$.
	Taking $\sigma$ given by Lemma \ref{l1}, $T=\mathcal{O}\left(n^{\frac{2}{1+2\alpha}}\right)$, and $\eta_t=\frac{2}{\mu(t+\kappa)}$, where $\kappa\geq\frac{2H^{1/\alpha}}{\mu}$, if $\zeta\in(\exp(-p/8),1)$, then with probability at least $1-\zeta$:
	\begin{equation*}
	\begin{aligned}
	&R(\hat{\theta}_n)-R(\theta^*) \\
	&\leq c_1\frac{G'\log(n)\sqrt[4]{p\log(1/\delta)}}{n^{\frac{\alpha}{1+2\alpha}}\epsilon^{1/2}}\left(1+\left(\frac{8\log(T/\zeta)}{p}\right)^{1/4}\right)^{\frac{1}{2}} \\
	&\quad+c_2\frac{G'^2\sqrt{p\log(1/\delta)}}{n^{\frac{2\alpha}{1+2\alpha}}\epsilon}\left(1+\left(\frac{8\log(T/\zeta)}{p}\right)^{1/4}\right) \\
	&\quad+c_3\frac{\log(n)}{n},
	\end{aligned}
	\end{equation*}
	for some constants $c_1,c_2$ and $c_3$, where $G'=\max\{2HM_\mathcal{C},H\}$.
\end{theorem}

Detailed proof can be found in Appendix A.4.
The proof is similar to Theorem \ref{t2}.
Considering that the stability parameter given by Lemma \ref{l9} is related to the optimization error, so the key is to obtain the optimization error.
The challenge is that the properties $G$-Lipschitz and $L$-smooth are replaced by the assumption $\alpha$-H{\"o}lder smooth when analyzing it.
To overcome the challenge, we use Lemma \ref{t1} to bound the optimization error and Young's inequality is used to normalize the exponential rate.
By choosing proper learning rate, we get an acceptable excess population risk bound.
Details are shown in the proof of Lemma \ref{l7}.

By Theorem \ref{t3}, it is easy to follow that with high probability, the excess population risk satisfies:
\begin{equation*}
R(\hat{\theta}_n)-R(\theta^*)=\mathcal{O}\left(\frac{p^{1/4}}{\epsilon^{1/2}}n^{\frac{-\alpha}{1+2\alpha}}\right).
\end{equation*}

\begin{remark}\label{r9}
	By the definition of $\alpha$-H{\"older} smooth, we have $\alpha\in(0,1]$, so our result is worse than the gradient based result given by \cite{wang2021differentially} ($\mathcal{O}(1/\sqrt{n})$ w.r.t $n$).
	One of the reasons is that \cite{wang2021differentially} does not consider the noise injected into the model when analyzing stability.
	However, as discussed before, the noise addition is independent to datasets, so we expand the stability to the noisy version in this paper, which generalize previous settings.
	Besides, our result is of $\mathcal{O}(\epsilon^{-1/2})$ w.r.t $\epsilon$, and previous results are of the order $\mathcal{O}(\epsilon^{-1})$.
	In practice, $\epsilon$ is always set less than $1$ to guarantee meaningful privacy, so our result is more superior when it comes to conditions that privacy requirements are strict (low $\epsilon$ conditions).
\end{remark}



Via the discussion mentioned above, we observe that the result given in Theorem \ref{t3} is far worse than which given in Theorem \ref{t2}, if we replace Lipschitzness and smoothness by $\alpha$-H{\"o}lder smoothness.
The reason is that when applying Young's inequality in the optimization error analysis, an additional term $\frac{H\eta_t^{\alpha+1}(1-\alpha)}{2(\alpha+1)}$ appears, leading a loose excess population risk bound.

Motivated by this, we design a variant of gradient perturbation method, called \textbf{max$\bf\{1,g\}$-Normalized Gradient Perturbation} DP algorithm, to overcome the loose excess population risk bound.
Details are shown in Algorithm \ref{alg1}.

\begin{algorithm}[tb]
	\caption{max$\bf\{1,g\}$-Normalized Gradient Perturbation}
	\label{alg1}
	\begin{algorithmic}[1]
		\State {\bfseries Input:} Dataset $D$, learning rate at iteration $t$: $\eta_t$, the variance of the Gaussian noise injected to the gradient: $\sigma$.
		\State Initialize $\theta_0$.
		\For{$t=0$ {\bfseries to} $T-1$}
		\State $G_t\leftarrow R_n(\hat{\theta}_t)+b_t$, $b_t\sim\mathcal{N}\left(0,\sigma^2I_p\right)$.		
		\If{$\left\Vert G_t\right\Vert_2<1$}
		\State $G_t\leftarrow G_t/\left\Vert G_t\right\Vert_2$.
		\EndIf
		\State $\hat{\theta}_{t+1}\leftarrow\hat{\theta}_{t}-\eta_{t}G_t$.
		\EndFor
		\State Return $\hat{\theta}_n=\hat{\theta}_T$.
	\end{algorithmic}
\end{algorithm}

\begin{remark}\label{r4}
	The difference between Algorithm \ref{alg1} and (\ref{DPGD}) is that in lines 5 and 6, we normalize the $\ell_2$-norm of the gradient to $1$ if it is less than $1$.
	In this way, we can `bypass' the Young's inequality when scaling $\|\theta_t-\theta_n^*\|_2^{1+\alpha}$ (derived from Lemma \ref{t1}), further remove term $\frac{H\eta_t^{\alpha+1}(1-\alpha)}{2(\alpha+1)}$ in the theoretical analysis.
	Details can be found in Appendix A.4.
\end{remark}

Then, we improve the excess population risk bound given in Theorem \ref{t3}.

\begin{theorem}\label{t4}
	If Assumptions \ref{a3}, \ref{a5} hold, the loss function and the parameter space are bounded, i.e. $0\leq\ell(\cdot,\cdot)\leq M_\ell$, $\|\mathcal{C}\|_2\leq M_\mathcal{C}$.
	Taking $\sigma$ given by Lemma \ref{l1}, $T=\mathcal{O}\left(\log(n)\right)$, and $\eta_1=\cdots=\eta_T=\eta$, where $\eta=\left(\frac{1}{H}\right)^{1/\alpha}$, if $\zeta\in(\exp(-p/8),1)$, then with probability at least $1-\zeta$,
	\begin{equation*}
	\begin{aligned}
	&R(\hat{\theta}_n)-R(\theta^*) \\
	&\leq c_1\frac{G'\log^{1.5}(n)\sqrt{p\log(1/\delta)}}{n\epsilon}\left(1+\left(\frac{8\log(T/\zeta)}{p}\right)^{1/4}\right) \\
	&\quad+c_2\frac{G'^2\log(n)p\log(1/\delta)}{n^2\epsilon^2}\left(1+\left(\frac{8\log(T/\zeta)}{p}\right)^{1/4}\right)^2 \\
	&\quad+c_3\frac{\log(n)}{n},
	\end{aligned}
	\end{equation*}
	for some constants $c_1,c_2,c_3>0$, where $G'=\max\{2HM_\mathcal{C},H\}$.
\end{theorem}

Detailed proof is given in Appendix A.5.
The proof is similar to Theorems \ref{t2} and \ref{t3}, the key difference is that by \textit{gradient normalization}, Young's inequality is abandoned in the theoretical analysis (as discussed in Remark \ref{r4}), which implies a better excess population risk bound.

\begin{remark}\label{r5}
	We introduce normalization from theoretical view, and the experiments show that it works in practice (details can be found in Section 5).
	Here, we provide some intuitions on why normalization works.
	For gradient perturbation method, it is easy to observe that the random noise $b_t$ is sampled independently at each iteration $t$, so it is reasonable to suppose $b_t$ does not change too much when it comes to different iterations.
	As a result, if the $\ell_2$-norm of the gradient is too small, the random noise will play a more important role and the gradient property will be concealed beneath the noise.
	Besides, the random noise itself would help the generalization error \cite{li2020on}.
	So in m-NGP, we apply normalization to scale the $\ell_2$-norm of the gradient to 1 if it is less than 1, strengthens the gradient along with the random noise.
\end{remark}

By Theorem \ref{t4}, it is easy to follow that with high probability, $R(\hat{\theta}_n)-R(\theta^*)=\mathcal{O}\big(\frac{\sqrt{p}}{n\epsilon}\big)$.
The bound is of the same order as the result given in Theorem \ref{t2}.
This is also the first $\mathcal{O}\left(1/n\right)$ high probability excess population risk bound over DP algorithm w.r.t $n$ without smoothness.



\begin{remark}\label{r3}
	Theorems \ref{t3} and \ref{t4} require Assumptions \ref{a3} and \ref{a5}, in this remark, we give some examples satisfy these assumptions.
	As discussed in Section 3, the $q$-norm hinge loss and $q$-th power absolute distance loss can be seemed as squared piecewise-linear loss functions when $q=2$, so they satisfy H{\"o}lder smoothness.
	For one-layer neural networks with squared error loss and leaky ReLU activations, the same phenomenon holds because the neural network can be seemed as a matrix multiplication.
	Beside, \cite{charles2018stability} shows that several interesting machine learning setups satisfy Assumption \ref{a5}, such as 1-layer neural networks with a squared error loss and leaky ReLU activations, loss functions of least squares minimizations, squared piecewise-linear functions with regularized term, etc.
	As a result, loss functions sataisfy those assumptions including but not limited to: (1) logistic regression and least squared minimization; (2) some of the squared piecewise linear functions; (3) some of the neural networks such as one-layer neural networks with squared error loss and leaky ReLU activations.
	The examples listed above are only part of the loss functions who satisfy those assumptions.
	We do not only focus on least squared minimization and logistic regression models, but extend the condition from smoothness to H{\"o}lder smoothness, and from strongly convex (convex) to the PL condition (some of the non-convex cases).
\end{remark}

\begin{figure*}[ht]
	\vskip 0.2in
	\begin{center}
		\centering{
			\subfigure[Iris]{\includegraphics[width=0.4\textwidth]{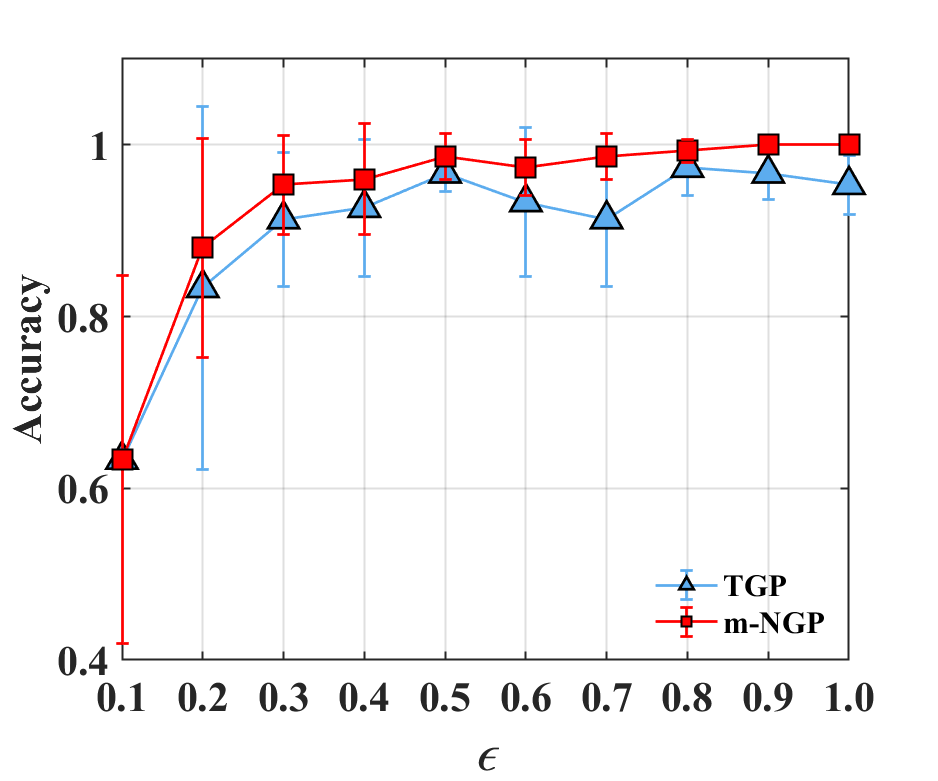}}
			\subfigure[Adult]{\includegraphics[width=0.4\textwidth]{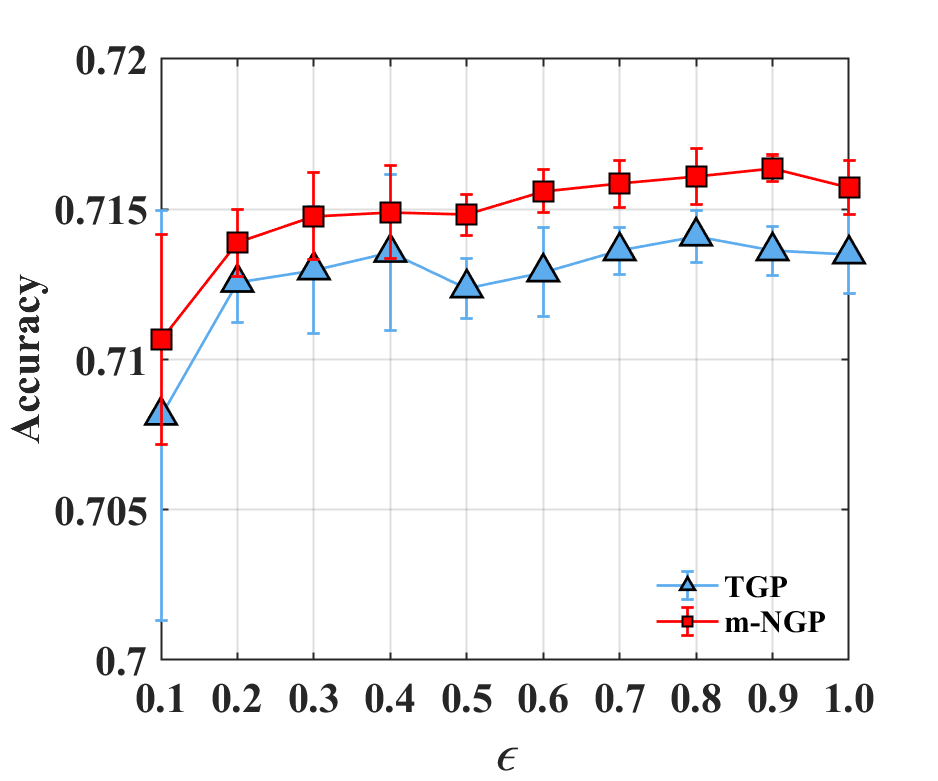}}
			\subfigure[Iris]{\includegraphics[width=0.4\textwidth]{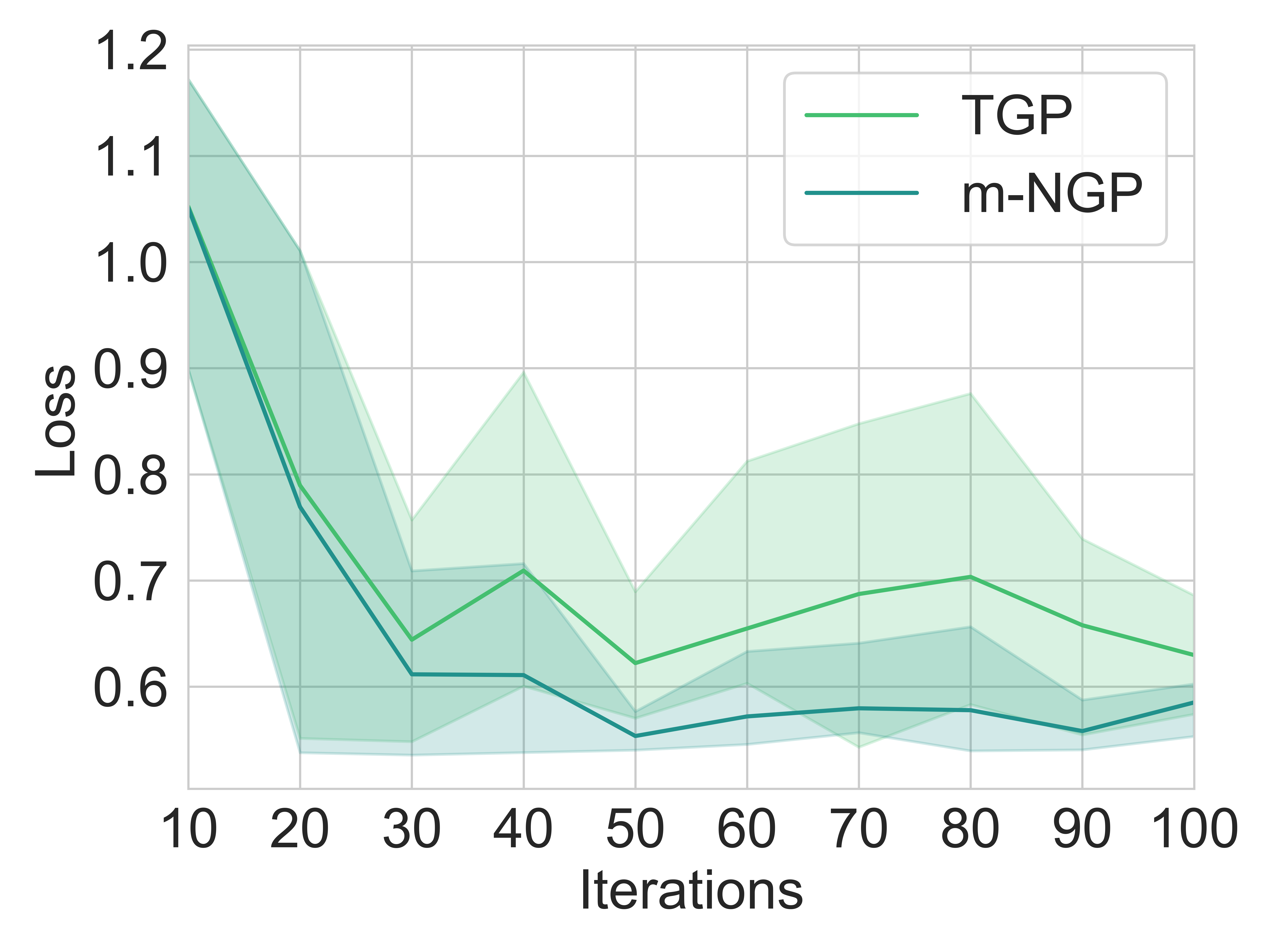}}
			\subfigure[Adult]{\includegraphics[width=0.4\textwidth]{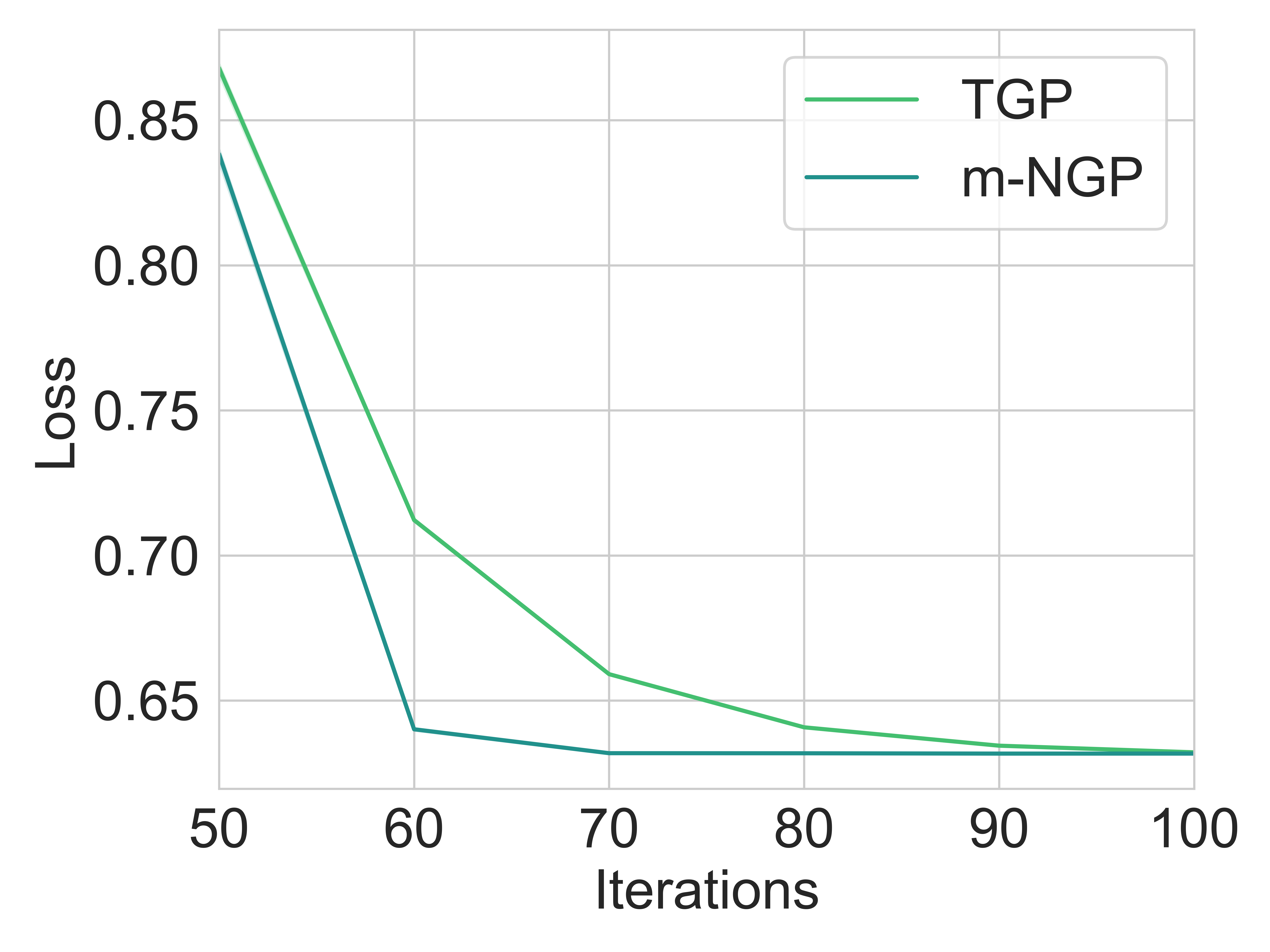}}
		}
	\end{center}
	\caption{Comparisons between Traditional Gradient Perturbation (TGP) and max$\bf\{1,g\}$-Normalized Gradient Perturbation (m-NGP).}
	\label{fig3}
	\vskip -0.2in
\end{figure*}

\section{5. Experiments}

In this section, we perform experiments on real datasets to evaluate our proposed our proposed m-NGP algorithm.

The experiments are performed on classification task over datasets Iris \cite{dua2017uci}, Breast Cancer \cite{mangasarian1990cancer}, Credit Card Fraud \cite{bontempi2018ulb}, Bank \cite{moro2014a}, and Adult \cite{dua2017uci}, the number of total data instances are 150, 699, 984, 41188, and 45222, respectively.
We split the training and testing sets randomly and evaluate the accuracy on the testing set and the convengence rate on the training set.
In all the experiments, the privacy budget $\delta$ is set $\frac{1}{n}$ and we choose $\epsilon=0.1$ to $1.0$.

We apply regularized logistic regression to the classification task, which satisfies the assumptions mentioned before, the experimental results are shown in Figure \ref{fig3}.
We show the results over datasets Iris and Adult in this section, experiments on other datasets are shown in Appendices B.1 and B.2.
For convergence rate, the shadow area represents the maximum and minimum loss over mutiple experiments, reflecting the variance.
The shadow area in part (d) of Figure \ref{fig3} is not obvious, the reason is that the variances are small.
Over most datasets, the accuracy and the convergence rate of m-NGP is better than traditional gradient perturbation method, which is in line with the theoretical analysis.
Moreover, in Appendix B.3, we perform experiments to demonstrate the effects brought by the dimension parameter $p$, the experimental results follow the theoretical results: with increasing $p$, the accuracy becomes worse in general.

\section{6. Conclusions}

In this paper, we first propose a state-of-the-art $\mathcal{O}\big(\frac{\sqrt{p}}{n\epsilon}\big)$ high probability excess population risk bound for gradient perturbation based DP algorithms, under the assumptions of $G$-Lipschitz, $L$-smooth, and Polyak-{\L}ojasiewicz condition.
The result positively answers the open problem: \textit{Can we achieve high probability excess risk bound with rate $\mathcal{O}(1/n)$ w.r.t $n$ for DP models via uniform stability?}
Then, we extend the result to a more general case, requiring $\alpha$-H{\"o}lder smoothness and Polyak-{\L}ojasiewicz condition.
However, the result is not as satisfactory as before, we achieve an $\mathcal{O}\big(n^{\frac{-\alpha}{1+2\alpha}}\big)$ high probability excess population risk bound, which cannot achieve an $\mathcal{O}\left(1/n\right)$ bound.
To get a better result, we further propose a new algorithm: max$\{1,g\}$-Normalized Gradient Perturbation (m-NGP).
Detailed theoretical analysis shows that m-NGP can achieve $\mathcal{O}\big(\frac{\sqrt{p}}{n\epsilon}\big)$ high probability excess population risk bound, under the assumptions of $\alpha$-H{\"o}lder smoothness and Polyak-{\L}ojasiewicz condition, which is the first $\mathcal{O}\left(1/n\right)$ high probability bound w.r.t $n$ under non-smoothness differentially private cases.
Experimental results show that the accuracy of m-NGP algorithm is better than traditional gradient perturbation method.
Thus, our proposed max$\{1,g\}$-Normalized Gradient Perturbation method improves the excess population risk bound and the accuracy of the DP model over real datasets, simultaneously.

\bibliographystyle{aaai}
\bibliography{icml2022}

\onecolumn
\begin{appendix}

\section{A. Details of Proofs}

\subsection{A.1. Proof of Lemma \ref{l9}}
\begin{lemma}\label{l9a}
	In gradient perturbation method (\ref{DPGD}), if Assumptions \ref{a1}, \ref{a5} hold, then
	\begin{equation*}
	\left\Vert\hat{\theta}_n-\hat{\theta}_n'\right\Vert_2\leq2\sqrt{2}\sqrt{\frac{R_n(\hat{\theta}_n)-R_n(\theta_n^*)}{\mu}}+\frac{4G}{\mu n},
	\end{equation*}
	where $\hat{\theta}_n$ and $\hat{\theta}_n'$ denote the models derived from adjacent datasets $S$ and $S'$.
\end{lemma}

\begin{proof}
	By triangle inequality, we have:
	\begin{equation}\label{e6}
	\left\Vert\hat{\theta}_n-\hat{\theta}_n'\right\Vert_2\leq\left\Vert\hat{\theta}_n-\theta_n^*(S)\right\Vert_2+\left\Vert\theta_n^*(S)-\theta_n^*(S')\right\Vert_2+\left\Vert\theta_n^*(S')-\hat{\theta}_n'\right\Vert_2.
	\end{equation}
	
	Recalling that PL condition implies QG condition:
	\begin{equation}\label{e2}
	R_n(\theta)-R_n(\theta_n^*)\geq\frac{\mu}{2}\left\Vert\theta-\theta_n^*\right\Vert_2^2.
	\end{equation}
	
	So we have:
	\begin{equation*}
	\begin{aligned}
	\left\Vert\hat{\theta}_n-\theta_n^*(S)\right\Vert_2&\leq\sqrt{\frac{2}{\mu}\left(R_n(\hat{\theta}_n)-R_n(\theta_n^*)\right)}, \\
	\left\Vert\theta_n^*(S')-\hat{\theta}_n'\right\Vert_2&\leq\sqrt{\frac{2}{\mu}\left(R_n(\hat{\theta}_n')-R_n(\theta_n^*(S'))\right)}.
	\end{aligned}
	\end{equation*}
	
	Due to the symmetry, the two inequalities can be integrated into one, i.e.
	\begin{equation}\label{e5}
	\begin{aligned}
	\left\Vert\hat{\theta}_n-\theta_n^*(S)\right\Vert_2+\left\Vert\theta_n^*(S')-\hat{\theta}_n'\right\Vert_2&\leq2\sqrt{2}\sqrt{\frac{R_n(\hat{\theta}_n)-R_n(\theta_n^*)}{\mu}}.
	\end{aligned}
	\end{equation}
	
	Now we turn to term $\left\Vert\theta_n^*(S)-\theta_n^*(S')\right\Vert_2$.
	
	For $R_n(\theta_n^*(S'))-R_n(\theta_n^*(S))$, if assuming the different data instance is $z_j$, we have:
	\begin{equation}\label{e1}
	\begin{aligned}
	R_n(\theta_n^*(S'))-R_n(\theta_n^*(S))&=\frac{1}{n}\left(\ell\left(z_j,\theta_n^*(S')\right)-\ell\left(z_j,\theta_n^*(S)\right)\right)+\frac{1}{n}\sum_{i\neq j}\ell\left(z_i,\theta_n^*(S')\right)-\ell\left(z_i,\theta_n^*(S)\right) \\
	&=\frac{1}{n}\left(\ell\left(z_j,\theta_n^*(S')\right)-\ell\left(z_j,\theta_n^*(S)\right)\right)+\frac{1}{n}\left(\ell\left(z_j',\theta_n^*(S')\right)-\ell\left(z_j',\theta_n^*(S)\right)\right) \\
	&\quad+R_{n'}(\theta_n^*(S'))-R_{n'}(\theta_n^*(S)) \\
	&\leq\frac{2G}{n}\left\Vert\theta_n^*(S')-\theta_n^*(S)\right\Vert_2+R_{n'}(\theta_n^*(S'))-R_{n'}(\theta_n^*(S)) \\
	&\leq\frac{2G}{n}\left\Vert\theta_n^*(S')-\theta_n^*(S)\right\Vert_2,
	\end{aligned}
	\end{equation}
	where $R_{n'}(\theta)$ is the empirical risk over dataset $S'$, the first inequality holds because $G$-Lipschitzness and the last inequality holds because $R_{n'}(\theta_n^*(S'))-R_{n'}(\theta_n^*(S))\leq0$ due to the definition of $\theta_n^*(S')$.
	
	By inequality (\ref{e2}),
	\begin{equation}\label{e3}
	R_n(\theta_n^*(S'))-R_n(\theta_n^*(S))\geq\frac{\mu}{2}\left\Vert\theta_n^*(S')-\theta_n^*(S)\right\Vert_2^2.
	\end{equation}
	
	Combining inequalities (\ref{e1}) and (\ref{e3}), we have:
	\begin{equation*}
	\frac{\mu}{2}\left\Vert\theta_n^*(S')-\theta_n^*(S)\right\Vert_2^2\leq\frac{2G}{n}\left\Vert\theta_n^*(S')-\theta_n^*(S)\right\Vert_2,
	\end{equation*}
	which implies
	\begin{equation}\label{e4}
	\left\Vert\theta_n^*(S')-\theta_n^*(S)\right\Vert_2\leq\frac{4G}{\mu n}.
	\end{equation}
	
	Plugging inequalities (\ref{e5}) and (\ref{e4}) back into (\ref{e6}), we have:
	\begin{equation*}
	\left\Vert\hat{\theta}_n-\hat{\theta}_n'\right\Vert_2\leq2\sqrt{2}\sqrt{\frac{R_n(\hat{\theta}_n)-R_n(\theta_n^*)}{\mu}}+\frac{4G}{\mu n},
	\end{equation*}
	which completes the proof.
	
\end{proof}

\subsection{A.2. The Optimization Error}

As discussed before, the excess population risk can be decomposed into:
\begin{equation}\label{EPR}
\begin{aligned}
R(\hat{\theta}_n)-R(\theta^*)&=R(\hat{\theta}_n)-R_n(\hat{\theta}_n)+R_n(\hat{\theta}_n)-R_n(\theta^*)+R_n(\theta^*)-R(\theta^*) \\
&\leq R(\hat{\theta}_n)-R_n(\hat{\theta}_n)+R_n(\hat{\theta}_n)-R_n(\theta_n^*)+R_n(\theta^*)-R(\theta^*).
\end{aligned}
\end{equation}

In Lemma \ref{l9}, the stability is also related to the optimization error (excess empirical risk), so we first discuss the optimization error of the private model $\hat{\theta}_n$, i.e. $R_n(\hat{\theta}_n)-R_n(\theta_n^*)$, under different assumptions.

To get the optimization error bound, we need the following lemma given in \cite{yang2021stability}.
\begin{lemma}[\cite{yang2021stability}]\label{l5}
	If Gaussian random noise $b\sim\mathcal{N}(0,\sigma^2I_p)$, then for $\zeta\in(\exp(-p/8),1)$, we have with probability $1-\zeta$,
	\begin{equation*}
	\|b\|_2\leq\sigma\sqrt{p}\left(1+\left(\frac{8\log(1/\zeta)}{p}\right)^{1/4}\right).
	\end{equation*}
\end{lemma}

\begin{lemma}\label{l6}
	If the Assumptions \ref{a1}, \ref{a2}, \ref{a5} hold and the DP model is trained by $T$-iterations gradient perturbation method (\ref{DPGD}), then taking $T=\mathcal{O}\left(\log(n)\right)$, $\eta_1=\cdots=\eta_T=\frac{1}{L}$, if $\zeta\in(\exp(-p/8),1)$, with probability at least $1-\zeta$,
	\begin{equation*}
	R_n(\hat{\theta}_n)-R_n(\theta_n^*)\lesssim\frac{G^2p\log(n)\log(1/\delta)}{n^2\epsilon^2}\left(1+\left(\frac{8\log(T/\zeta)}{p}\right)^{1/4}\right)^2.
	\end{equation*}
\end{lemma}

\begin{proof}
	
	Note that we assume the loss function is $L$-smooth (Assumption \ref{a2}, denoted by $L$) and satisfies the PL condition (Assumption \ref{a5}, denoted by $PL$), at iteration $t$, taking $\eta_t=\frac{1}{L}$, we have:
	\begin{equation}\label{E1}
	\begin{aligned}
	R_n(\hat{\theta}_{t+1})-R_n(\hat{\theta}_t)&\overset{(L)}{\leq}\langle\nabla_\theta R_n(\hat{\theta}_t),\hat{\theta}_{t+1}-\hat{\theta}_t\rangle+\frac{L}{2}\left\Vert\hat{\theta}_{t+1}-\hat{\theta}_t\right\Vert_2^2 \\
	&=-\eta_t\langle\nabla_\theta R_n(\hat{\theta}_t),\nabla_\theta R_n(\hat{\theta}_t)+b_{t+1}\rangle+\frac{L\eta_t^2}{2}\left\Vert\nabla_\theta R_n(\hat{\theta}_t)+b_{t+1}\right\Vert_2^2 \\
	&=-\frac{1}{L}\left\Vert\nabla_\theta R_n(\hat{\theta}_t)\right\Vert_2^2+\frac{1}{2L}\left\Vert\nabla_\theta R_n(\hat{\theta}_t)\right\Vert_2^2+\frac{1}{2L}\|b\|_2^2 \\
	&=-\frac{1}{2L}\left\Vert\nabla_\theta R_n(\hat{\theta}_t)\right\Vert_2^2+\frac{1}{2L}\|b_{t+1}\|_2^2 \\
	&\overset{(PL)}{\leq}-\frac{\mu}{L}\left(R_n(\hat{\theta}_t)-R_n(\theta_n^*)\right)+\frac{1}{2L}\|b_{t+1}\|_2^2.
	\end{aligned}
	\end{equation}
	
	Adding $R_n(\hat{\theta}_t)-R_n(\theta_n^*)$ to both sides, we have:
	\begin{equation*}
	R_n(\hat{\theta}_{t+1})-R_n(\theta_n^*)\leq\left(1-\frac{\mu}{L}\right)\left(R_n(\hat{\theta}_t)-R_n(\theta_n^*)\right)+\frac{1}{2L}\|b_{t+1}\|_2^2.
	\end{equation*}
	
	Summing over $T$ iterations, we have:
	\begin{equation}\label{OEnonoise}
	\begin{aligned}
	R_n(\hat{\theta}_n)-R_n(\theta_n^*)&\leq\left(1-\frac{\mu}{L}\right)^T\left(R_n(\hat{\theta}_0)-R_n(\theta_n^*)\right)+\frac{1}{2L}\sum_{t=0}^{T-1}\left(1-\frac{\mu}{L}\right)^t\|b_{t+1}\|_2^2.
	\end{aligned}
	\end{equation}
	
	With Lemma \ref{l5}, with probability at least $1-\xi$, we have:
	\begin{equation*}
	\|b_t\|_2^2\leq\sigma^2p\left(1+\left(\frac{8\log(T/\xi)}{p}\right)^{1/4}\right)^2,
	\end{equation*}
	for all $t=1,\cdots,T$.
	
	Then, with the high probability upper bound of $b_t$, inequality (\ref{OEnonoise}) comes to:
	\begin{equation}
	\begin{aligned}
	R_n(\hat{\theta}_n)-R_n(\theta_n^*)&\leq\left(1-\frac{\mu}{L}\right)^T\left(R_n(\hat{\theta}_0)-R_n(\theta_n^*)\right)+\frac{1}{2L}\sum_{t=0}^{T-1}\left(1-\frac{\mu}{L}\right)^t\sigma^2p\left(1+\left(\frac{8\log(T/\zeta)}{p}\right)^{1/4}\right)^2 \\
	&\leq\left(1-\frac{\mu}{L}\right)^TM_\ell+\frac{\left(1-\left(1-\frac{\mu}{L}\right)^{T-1}\right)\sigma^2p}{2\mu}\left(1+\left(\frac{8\log(T/\zeta)}{p}\right)^{1/4}\right)^2 \\
	&\leq\left(1-\frac{\mu}{L}\right)^TM_\ell+\frac{\sigma^2p}{2\mu}\left(1+\left(\frac{8\log(T/\zeta)}{p}\right)^{1/4}\right)^2,
	\end{aligned}
	\end{equation}
	where the second inequality holds because $0\leq\ell(\cdot,\cdot)\leq M_\ell$, and $0<\frac{\mu}{L}<1$ because of the definitions of $\mu$ and $L$ \cite{wang2017differentially}.
	
	Taking $\sigma=c\frac{G\sqrt{T\log(1/\delta)}}{n\epsilon}$ given in Lemma \ref{l1}, and taking $T=\mathcal{O}(\log(n))$, then if $\zeta\in(\exp(-p/8),1)$, with probability at least $1-\zeta$, we have:
	\begin{equation*}
	R_n(\hat{\theta}_n)-R_n(\theta_n^*)\lesssim\frac{G^2p\log(n)\log(1/\delta)}{2\mu n^2\epsilon^2}\left(1+\left(\frac{8\log(T/\zeta)}{p}\right)^{1/4}\right)^2.
	\end{equation*}
	
	The result follows.
	
\end{proof}

\begin{lemma}\label{l7}
	If the loss function is $\alpha$-H{\"o}lder smooth with parameter $H$, satisfies the PL inequality with parameter $2\mu$ and the DP model is trained by $T$-iterations gradient perturbation method (\ref{DPGD}), then taking $T=\mathcal{O}\left(n^{\frac{2}{1+2\alpha}}\right)$, $\eta_t=\frac{2}{\mu(t+\kappa)}$, where $\kappa\geq\frac{2H^{1/\alpha}}{\mu}$, if $\zeta\in(\exp(-p/8),1)$, with probability at least $1-\zeta$,
	\begin{equation*}
	R_n(\hat{\theta}_n)-R_n(\theta_n^*)\lesssim\frac{G'^2\sqrt{p\log(1/\delta)}}{n^{\frac{2\alpha}{1+2\alpha}}\epsilon}\left(1+\left(\frac{8\log(T/\zeta)}{p}\right)^{1/4}\right).
	\end{equation*}
\end{lemma}

\begin{proof}
	The proof is motivated by \cite{li2021improved}.
	
	Like the proof of Lemma \ref{l6}, by assuming that the loss function is $\alpha$-H{\"o}lder smooth (Assumption \ref{a3}, denoted by $\alpha$), via Lemma \ref{t1}, at iteration $t$,
	\begin{equation*}
	\begin{aligned}
	R_n(\hat{\theta}_{t+1})-R_n(\hat{\theta}_t)&\overset{(\alpha)}{\leq}\langle\nabla_\theta R_n(\hat{\theta}_t),\hat{\theta}_{t+1}-\hat{\theta}_t\rangle+\frac{H}{2}\left\Vert\hat{\theta}_{t+1}-\hat{\theta}_t\right\Vert_2^{\alpha+1} \\
	&\leq\langle\nabla_\theta R_n(\hat{\theta}_t),\hat{\theta}_{t+1}-\hat{\theta}_t\rangle+\frac{H}{\alpha+1}\left\Vert\hat{\theta}_{t+1}-\hat{\theta}_t\right\Vert_2^{\alpha+1} \\
	&=-\eta_t\langle\nabla_\theta R_n(\hat{\theta}_t),\nabla_\theta R_n(\hat{\theta}_t)+b_{t+1}\rangle+\frac{H\eta_t^{\alpha+1}}{\alpha+1}\left\Vert\nabla_\theta R_n(\hat{\theta}_t)+b_{t+1}\right\Vert_2^{\alpha+1} \\
	&\leq-\eta_t\left(\left\Vert\nabla_\theta R_n(\hat{\theta}_t)\right\Vert_2^2+\langle\nabla_\theta R_n(\hat{\theta}_t),b_{t+1}\rangle\right) \\
	&\quad+\frac{H\eta_t^{\alpha+1}}{\alpha+1}\left(\frac{1-\alpha}{2}+\frac{\alpha+1}{2}\left(\left\Vert\nabla_\theta R_n(\hat{\theta}_t)+b_{t+1}\right\Vert_2^{\alpha+1}\right)^{\frac{2}{\alpha+1}}\right) \\
	&\leq-\eta_t\left\Vert\nabla_\theta R_n(\hat{\theta}_t)\right\Vert_2^2+\left(H\eta_t^{\alpha+1}-\eta_t\right)\langle\nabla_\theta R_n(\hat{\theta}_t),b_{t+1}\rangle \\
	&\quad+\frac{H\eta_t^{\alpha+1}(1-\alpha)}{2(\alpha+1)}+\frac{H\eta_t^{\alpha+1}}{2}\left(\left\Vert\nabla_\theta R_n(\hat{\theta}_t)\right\Vert_2^2+\|b_{t+1}\|_2^2\right),
	\end{aligned}
	\end{equation*}
	where the third inequality is because of Young's inequality: if $p^{-1}+q^{-1}=1$ and $p>0$, then $uv\leq p^{-1}|u|^p+q^{-1}|v|^q$.
	Here we set $p^{-1}=(1-\alpha)/2$, $q^{-1}=(\alpha+1)/2$.
	And the last inequality holds because of Cauthy-Schwarz inequality.
	
	Noting that $\eta_t=\frac{2}{\mu(t+\kappa)}$ and $\kappa\geq\frac{2H^{1/\alpha}}{\mu}$, so we have: $\eta_t\leq\left(\frac{1}{H}\right)^{1/\alpha}$.
	
	As a result, we have:
	\begin{equation}\label{etaleq}
	H\eta_t^{\alpha+1}\leq H\left[\left(\frac{1}{H}\right)^{1/\alpha}\right]^\alpha\eta_t\leq\eta_t.
	\end{equation}
	
	As a result,
	\begin{equation}\label{HSt}
	\begin{aligned}
	R_n(\hat{\theta}_{t+1})-R_n(\hat{\theta}_t)&\leq-\eta_t\left\Vert\nabla_\theta R_n(\hat{\theta}_t)\right\Vert_2^2+\frac{H\eta_t^{\alpha+1}(1-\alpha)}{2(\alpha+1)}+\frac{H\eta^{\alpha+1}}{2}\left\Vert\nabla_\theta R_n(\hat{\theta}_t)\right\Vert_2^2 \\
	&\quad+\eta_t\|\nabla_\theta R_n(\hat{\theta}_t)\|_2\|b_{t+1}\|_2+\frac{H\eta_t^{\alpha+1}}{2}\|b_{t+1}\|_2^2 \\
	&\leq-\frac{\eta_t}{2}\left\Vert\nabla_\theta R_n(\hat{\theta}_t)\right\Vert_2^2+\frac{H\eta_t^{\alpha+1}(1-\alpha)}{2(\alpha+1)}+2\eta_t\|\nabla_\theta R_n(\hat{\theta}_t)\|_2\|b_{t+1}\|_2+\frac{\eta_t}{2}\|b_{t+1}\|_2^2 \\
	&\overset{(PL)}{\leq}-\frac{\eta_t}{4}\left\Vert\nabla_\theta R_n(\hat{\theta}_t)\right\Vert_2^2-\mu\eta_t\left(R_n(\hat{\theta}_t)-R_n(\theta_n^*)\right)+\frac{H\eta_t^{\alpha+1}(1-\alpha)}{2(\alpha+1)} \\
	&\quad+\eta_tG'\|b_{t+1}\|_2+\frac{\eta_t}{2}\|b_{t+1}\|_2^2,
	\end{aligned}
	\end{equation}
	where the second inequality is because of (\ref{etaleq}) and the last inequality holds because we assume that the loss function satisfies the PL condition with parameter $2\mu$, and $G'=\max\{2HM_\mathcal{C},H\}$, as discussed in Lemma \ref{l1}.
	
	Adding $\frac{\eta_t}{4}\left\Vert\nabla_\theta R_n(\hat{\theta}_t)\right\Vert_2^2-R_n(\theta_n^*)$ to both sides of (\ref{HSt}), we have
	\begin{equation*}
	\begin{aligned}
	R_n(\hat{\theta}_{t+1})-R_n(\theta_n^*)+\frac{\eta_t}{4}\left\Vert\nabla_\theta R_n(\hat{\theta}_t)\right\Vert_2^2&\leq\left(1-\mu\eta_t\right)\left(R_n(\hat{\theta}_t)-R_n(\theta_n^*)\right)+\frac{H\eta_t^{\alpha+1}(1-\alpha)}{2(\alpha+1)} \\
	&\quad+\eta_tG'\|b_{t+1}\|_2+\frac{\eta_t}{2}\|b_{t+1}\|_2^2.
	\end{aligned}
	\end{equation*}
	
	Taking $\eta_t=\frac{2}{\mu(t+\kappa)}$,
	\begin{equation*}
	\begin{aligned}
	R_n(\hat{\theta}_{t+1})-R_n(\theta_n^*)+\frac{1}{2\mu(t+\kappa)}\left\Vert\nabla_\theta R_n(\hat{\theta}_t)\right\Vert_2^2&\leq\frac{t+\kappa-2}{t+\kappa}\left(R_n(\hat{\theta}_t)-R_n(\theta_n^*)\right)+\frac{H\eta_t^{\alpha+1}(1-\alpha)}{2(\alpha+1)} \\
	&\quad+\eta_tG'\|b_{t+1}\|_2+\frac{\eta_t}{2}\|b_{t+1}\|_2^2.
	\end{aligned}
	\end{equation*}
	
	Multiply both side by $(t+\kappa)(t+\kappa-1)$,
	\begin{equation}\label{E2}
	\begin{aligned}
	&(t+\kappa)(t+\kappa-1)\left(R_n(\hat{\theta}_{t+1})-R_n(\theta_n^*)\right)+\frac{t+\kappa-1}{2\mu}\left\Vert\nabla_\theta R_n(\hat{\theta}_t)\right\Vert_2^2 \\
	&\leq(t+\kappa-1)(t+\kappa-2)\left(R_n(\hat{\theta}_t)-R_n(\theta_n^*)\right)+\left(t+\kappa\right)^{-\alpha}(t+\kappa-1)\frac{H(1-\alpha)}{2(\alpha+1)}\left(\frac{2}{\mu}\right)^{1+\alpha} \\
	&\quad+\frac{2G'(t+\kappa-1)}{\mu}\|b_{t+1}\|_2+\frac{t+\kappa-1}{\mu}\|b_{t+1}\|_2^2.
	\end{aligned}
	\end{equation}
	
	With Lemma \ref{l5}, for each $t=1,\cdots,T$, with probability at least $1-\xi$, we have:
	\begin{equation*}
	\begin{aligned}
	\|b_t\|_2&\leq\sigma\sqrt{p}\left(1+\left(\frac{8\log(1/\xi)}{p}\right)^{1/4}\right), \\
	\|b_t\|_2^2&\leq\sigma^2p\left(1+\left(\frac{8\log(1/\xi)}{p}\right)^{1/4}\right)^2.
	\end{aligned}
	\end{equation*}
	
	So with probability at least $1-\xi$:
	\begin{equation*}
	\begin{aligned}
	&(t+\kappa)(t+\kappa-1)\left(R_n(\hat{\theta}_{t+1})-R_n(\theta_n^*)\right)+\frac{t+\kappa-1}{2\mu}\left\Vert\nabla_\theta R_n(\hat{\theta}_t)\right\Vert_2^2 \\
	&\leq(t+\kappa-1)(t+\kappa-2)\left(R_n(\hat{\theta}_t)-R_n(\theta_n^*)\right)+\left(t+\kappa\right)^{-\alpha}(t+\kappa-1)\frac{H(1-\alpha)}{2(\alpha+1)}\left(\frac{2}{\mu}\right)^{1+\alpha} \\
	&\quad+\frac{2G'(t+\kappa-1)\sigma\sqrt{p}}{\mu}\left(1+\left(\frac{8\log(1/\xi)}{p}\right)^{1/4}\right)+\frac{(t+\kappa-1)\sigma^2p}{\mu}\left(1+\left(\frac{8\log(1/\xi)}{p}\right)^{1/4}\right)^2.
	\end{aligned}
	\end{equation*}
	
	By summing over $T$ iterations and taking $\xi=\zeta/T$, with probability at least $1-\zeta$, we have:
	\begin{equation}\label{HST}
	\begin{aligned}
	&(T+\kappa)(T+\kappa-1)\left(R_n(\hat{\theta}_{T+1})-R_n(\theta_n^*)\right)+\sum_{t=1}^{T}\frac{t+\kappa-1}{2\mu}\left\Vert\nabla_\theta R_n(\hat{\theta}_t)\right\Vert_2^2 \\
	&\leq\kappa(\kappa-1)\left(R_n(\hat{\theta}_1)-R_n(\theta_n^*)\right)+\sum_{t=1}^{T}\left(t+\kappa\right)^{-\alpha}(t+\kappa-1)\frac{H(1-\alpha)}{2(\alpha+1)}\left(\frac{2}{\mu}\right)^{1+\alpha} \\
	&\quad+\sum_{t=1}^{T}\frac{2G'(t+\kappa-1)\sigma\sqrt{p}}{\mu}\left(1+\left(\frac{8\log(T/\zeta)}{p}\right)^{1/4}\right) \\
	&\quad+\sum_{t=1}^T\frac{(t+\kappa-1)\sigma^2p}{\mu}\left(1+\left(\frac{8\log(T/\zeta)}{p}\right)^{1/4}\right)^2.
	\end{aligned}
	\end{equation}
	
	Here, for simplicity, we represent $t=0,\cdots,T-1$ by $t=1,\cdots,T$.
	
	Then we bound term $\sum_{t=1}^{T}\left(t+\kappa\right)^{-\alpha}(t+\kappa-1)\frac{H(1-\alpha)}{2(\alpha+1)}\left(\frac{2}{\mu}\right)^{1+\alpha}$.
	
	Note that:
	\begin{equation*}
	\sum_{t=1}^{T}\left(t+\kappa\right)^{-\alpha}(t+\kappa-1)\leq\sum_{t=1}^{T}\left(t+\kappa\right)^{1-\alpha}\leq\int_1^T(t+\kappa)^{1-\alpha}dt\leq\frac{(T+\kappa)^{2-\alpha}}{2-\alpha}.
	\end{equation*}
	
	Plugging the result above back into (\ref{HST}), and note that $0\leq\ell(\cdot,\cdot)\leq M_\ell$, we have:
	\begin{equation*}
	\begin{aligned}
	(T+\kappa)(T+\kappa-1)\left(R_n(\hat{\theta}_{T+1})-R_n(\theta_n^*)\right)&\leq\kappa(\kappa-1)M_\ell+\frac{(T+\kappa)^{2-\alpha}}{2-\alpha}\frac{H(1-\alpha)}{2(\alpha+1)}\left(\frac{2}{\mu}\right)^{1+\alpha} \\
	&\quad+\left(T\kappa+\frac{T(T-1)}{2}\right)\frac{2G'\sigma\sqrt{p}}{\mu}\left(1+\left(\frac{8\log(T/\zeta)}{p}\right)^{1/4}\right) \\
	&\quad+\left(T\kappa+\frac{T(T-1)}{2}\right)\frac{\sigma^2p}{\mu}\left(1+\left(\frac{8\log(T/\zeta)}{p}\right)^{1/4}\right)^2.
	\end{aligned}
	\end{equation*}
	
	As a result, taking $\sigma$ given in Lemma \ref{l1}, with probability at least $1-\zeta$, we have:
	\begin{equation*}
	R_n(\hat{\theta}_{T+1})-R_n(\theta_n^*)\lesssim T^{-\alpha}+\frac{G'^2\sqrt{Tp\log(1/\delta)}}{n\epsilon}\left(1+\left(\frac{8\log(T/\zeta)}{p}\right)^{1/4}\right).
	\end{equation*}
	
	Taking $T=\mathcal{O}\left(n^{\frac{2}{1+2\alpha}}\right)$, with probability at least $1-\zeta$, we have:
	\begin{equation*}
	R_n(\hat{\theta}_n)-R_n(\theta_n^*)\lesssim\frac{G'^2\sqrt{p\log(1/\delta)}}{n^{\frac{2\alpha}{1+2\alpha}}\epsilon}\left(1+\left(\frac{8\log(T/\zeta)}{p}\right)^{1/4}\right).
	\end{equation*}
	
	The result follows.
	
\end{proof}

\subsection{A.3. Proof of Theorem \ref{t2}}

Before the detailed proof, we first prove the following lemma \ref{l2}.
To get Lemma \ref{l2}, we need the following lemmas given in \cite{bousquet2020sharper}.
\begin{lemma}[\cite{bousquet2020sharper}]\label{l3}
	Assume that $z_1,\cdots,z_n$ are independent variables and the function $g_i:\mathcal{Z}^n\rightarrow\mathbb{R}$ satisfy the following properties for $i=1,\cdots,n$,
	\begin{itemize}
		\item $\mathbb{E}_{z_i}g_i(z_1,\cdots,z_n)=0$ almost surely;
		\item $\left|\mathbb{E}\left[g_i(z_1,\cdots,z_n)|z_i\right]\right|\leq K$ almost surely;
		\item $\left|g_i(z_1,\cdots,z_n)-g_i(z_1,\cdots,z_{j-1},z_j',z_{j+1},\cdots,z_n)\right|\leq\beta$.
	\end{itemize}
	Then the following inequality holds for all $q\geq2$,
	\begin{equation*}
	\left\Vert\sum_{i=1}^{n}g_i\right\Vert_q\leq12\sqrt{2}\beta qn\log(n)+4K\sqrt{qn}.
	\end{equation*}
\end{lemma}

\begin{lemma}[\cite{bousquet2020sharper}]\label{l8}
	Under the uniform stability condition with parameter $\gamma$ and uniformly bounded loss function $\ell(\cdot,\cdot)\leq M_\ell$, we have for $g_i=\mathbb{E}_{z_i'}\left(\ell(z_i,\theta_n^{(i)})-\mathbb{E}_z\ell(z,\theta_n^{(i)})\right)$,
	\begin{equation*}
	\left|n\left(R_n(\theta_n)-R(\theta_n)\right)-\sum_{i=1}^{n}g_i\right|\leq2\gamma n.
	\end{equation*}
\end{lemma}

\begin{lemma}\label{l2}
	Defining the DP algorithm (model) training by $T$-iterations gradient perturbation method (like (\ref{DPGD})) $\hat{\theta}_n=\hat{\theta}(z_1,\cdots,z_n)$  and its independent copy $\hat{\theta}_n'=\hat{\theta}(z_1',\cdots,z_n')$.
	Then for all $q\geq2$,
	\begin{equation*}
	\left\Vert R_n(\hat{\theta}_n)-R(\hat{\theta}_n)-\frac{1}{n}\sum_{i=1}^{n}\mathbb{E}\left[\ell(z_i,\hat{\theta}_n')|z_i\right]+\mathbb{E} R(\hat{\theta}_n)\right\Vert_q\lesssim Gq\log(n)\left(\sqrt{\frac{R_n(\hat{\theta}_n)-R_n(\theta_n^*)}{\mu}}+\frac{G}{\mu n}\right).
	\end{equation*}
\end{lemma}

\begin{proof}
	Via Lemma \ref{l9a},
	\begin{equation*}
	\left\Vert\hat{\theta}_n-\hat{\theta}_n'\right\Vert_2\leq2\sqrt{2}\sqrt{\frac{R_n(\hat{\theta}_n)-R_n(\theta_n^*)}{\mu}}+\frac{4G}{\mu n}.
	\end{equation*}
	
	Recalling the definition of $\gamma$-uniformly stability:
	If for any $z,z',z_1,\cdots,z_n\in\mathcal{Z}$ and $i=1,\cdots,n$, it holds that
	\begin{equation*}
	\left|\ell(z,\theta_n\left(z_1,\cdots,z_n\right))-\ell(z,\theta_n\left(z_1,\cdots,z_{i-1},z',z_{i+1},\cdots,z_n\right))\right|\leq\gamma.
	\end{equation*}
	
	Then with $G$-Lipschitzness, we have:
	\begin{equation*}
	\left|\ell(z,\hat{\theta}_n)-\ell(z,\hat{\theta}_n')\right|\leq2\sqrt{2}G\sqrt{\frac{R_n(\hat{\theta}_n)-R_n(\theta_n^*)}{\mu}}+\frac{4G^2}{\mu n},
	\end{equation*}
	where $\hat{\theta}_n$ and $\hat{\theta}_n'$ are private models derived from any adjacent datasets.
	In the following, we use $\hat{\theta}_n^{(i)}$ to represent $\hat{\theta}_n'$, which means that the single different data instance is the $i^{th}$ one.
	
	Considering the function $g_i(z_1,\cdots,z_n)=\mathbb{E}_{z_i'}[\ell(z_i,\hat{\theta}_n^{(i)})]-\mathbb{E}_{z_i'}[R(\hat{\theta}_n^{(i)})]$, via the definition of $R(\hat{\theta}^{(i)})$, we have: $\mathbb{E}_{z_i}g_i(z_1,\cdots,z_n)=0$.
	
	With the stability of the DP model, we have:
	\begin{equation*}
	\left|g_i(z_1,\cdots,z_n)-g_i(z_1,\cdots,z_{j-1},z_j',z_{j+1},\cdots,z_n)\right|\leq\beta\coloneqq2\left(2\sqrt{2}G\sqrt{\frac{R_n(\hat{\theta}_n)-R_n(\theta_n^*)}{\mu}}+\frac{4G^2}{\mu n}\right).
	\end{equation*}
	
	If considering $h_i(z_1,\cdots,z_n)=g_i(z_1,\cdots,z_n)-\mathbb{E}[g_i(z_1,\cdots,z_n)|z_i]$, we have:
	\begin{equation*}
	\mathbb{E}_{z_i}h_i(z_1,\cdots,z_n)=0
	\end{equation*}
	almost surely, and
	\begin{equation*}
	\left|h_i(z_1,\cdots,z_n)-h_i(z_1,\cdots,z_{j-1},z_j',z_{j+1},\cdots,z_n)\right|\leq2\beta=4\left(2\sqrt{2}G\sqrt{\frac{R_n(\hat{\theta}_n)-R_n(\theta_n^*)}{\mu}}+\frac{4G^2}{\mu n}\right).
	\end{equation*}
	
	Via the definition of $h_i$, we observe that $\mathbb{E}\left[h_i|z_i\right]=0$ almost surely, which further implies $K=0$ in Lemma \ref{l3}, so we have for $q\geq2$:
	\begin{equation*}
	\left\Vert\sum_{i=1}^{n}h_i\right\Vert_q=\left\Vert\sum_{i=1}^{n}\left(g_i-\mathbb{E}\left[g_i|z_i\right]\right)\right\Vert_q\leq192Gqn\log(n)\left(\sqrt{\frac{R_n(\hat{\theta}_n)-R_n(\theta_n^*)}{2\mu}}+\frac{G}{\mu n}\right).
	\end{equation*}
	
	Via Lemma \ref{l8}, we have:
	\begin{equation*}
	\left|n\left(R_n(\hat{\theta}_n)-R(\hat{\theta}_n)\right)-\sum_{i=1}^{n}g_i\right|\leq2n\left(2\sqrt{2}G\sqrt{\frac{R_n(\hat{\theta}_n)-R_n(\theta_n^*)}{\mu}}+\frac{4G^2}{\mu n}\right).
	\end{equation*}
	
	Noting that
	\begin{equation*}
	\mathbb{E}\left[g_i|z_i\right]=\mathbb{E}\left[\ell(z_i,\hat{\theta}_n')|z_i\right]-\mathbb{E} R(\hat{\theta}_n'),
	\end{equation*}
	we have:
	\begin{equation*}
	\left\Vert R_n(\hat{\theta}_n)-R(\hat{\theta}_n)-\frac{1}{n}\sum_{i=1}^{n}\mathbb{E}\left[\ell(z_i,\hat{\theta}_n')|z_i\right]+\mathbb{E}R(\hat{\theta}_n)\right\Vert_q\lesssim Gq\log(n)\left(\sqrt{\frac{R_n(\hat{\theta}_n)-R_n(\theta_n^*)}{\mu}}+\frac{G}{\mu n}\right).
	\end{equation*}
	
	The result follows.
	
\end{proof}

To get Theorem \ref{t2}, we further need the following lemma given in \cite{boucheron2013concentration}.
\begin{lemma}[\cite{boucheron2013concentration}]\label{l4}
	If $X_1,\cdots,X_n$ are zero mean, independent and bounded $|X_i|\leq M$ almost surely, then for $q\geq2$,
	\begin{equation*}
	\left\Vert X_1+\cdots X_n\right\Vert_q\leq6\sqrt{\left(\sum_{i=1}^{n}\mathbb{E}[X_i^2]\right)q}+4qM.
	\end{equation*}
\end{lemma}

Then, we can start our proof.
\begin{theorem}
	If Assumptions \ref{a1}, \ref{a2}, and \ref{a5} hold, the loss function is bounded, i.e. $0\leq\ell(\cdot,\cdot)\leq M_\ell$, taking $\sigma$ given by Lemma \ref{l1}, $T=\mathcal{O}\left(\log(n)\right)$, $\eta_1=\cdots=\eta_T=\frac{1}{L}$, if $\zeta\in(\exp(-p/8),1)$, then with probability at least $1-\zeta$:
	\begin{equation*}
	\begin{aligned}
	R(\hat{\theta}_n)-R(\theta^*)&\leq c_1\frac{G^2p\log(n)\log(1/\delta)}{n^2\epsilon^2}\left(1+\left(\frac{8\log(T/\zeta)}{p}\right)^{1/4}\right)^2 \\
	&\quad+c_2\frac{G\log^{1.5}(n)\sqrt{p\log(1/\delta)}}{n\epsilon}\left(1+\left(\frac{8\log(T/\zeta)}{p}\right)^{1/4}\right)+c_3\frac{\log(n)}{n}.
	\end{aligned}
	\end{equation*}
	for some constants $c_1,c_2,c_3>0$.
\end{theorem}

\begin{proof}
	
	Via Lemma \ref{l2}, we have:
	\begin{equation*}
	R_n(\hat{\theta}_n)-R(\hat{\theta}_n)=\rho+\frac{1}{n}\sum_{i=1}^{n}\mathbb{E}'\ell(z_i,\hat{\theta}_n')-\mathbb{E}R(\hat{\theta}_n),
	\end{equation*}
	where $\|\rho\|_q\lesssim Gq\log(n)\left(\sqrt{\frac{R_n(\hat{\theta}_n)-R_n(\theta_n^*)}{\mu}}+\frac{G}{\mu n}\right)$ for $q\geq2$ and $\mathbb{E}'$ denotes the expectation taken over the independent copy.
	
	Plugging this back to (\ref{EPR}), we have:
	\begin{equation*}
	R(\hat{\theta}_n)-R(\theta^*)\leq\left(R_n(\hat{\theta}_n)-R_n(\theta_n^*)\right)+\left(R_n(\theta^*)-R(\theta^*)\right)-\rho-\frac{1}{n}\sum_{i=1}^{n}\mathbb{E}'\ell(z_i,\hat{\theta}_n')+\mathbb{E}R(\hat{\theta}_n).
	\end{equation*}
	
	Noting that $R_n(\theta^*)=\frac{1}{n}\sum_{i=1}^{n}\ell(z_i,\theta^*)$, we have:
	\begin{equation}\label{EPRE'}
	R(\hat{\theta}_n)-R(\theta^*)\leq\left(R_n(\hat{\theta}_n)-R_n(\theta_n^*)\right)+\left(\mathbb{E}R(\hat{\theta}_n)-R(\theta^*)\right)-\rho-\frac{1}{n}\sum_{i=1}^{n}\left(\mathbb{E}'\ell(z_i,\hat{\theta}_n')-\ell(z_i,\theta^*)\right).
	\end{equation}
	
	Based on the definition of $R(\theta)$, Assumption \ref{a4} is equivalent to:
	\begin{equation}\label{equGBC}
	\mathbb{E}\left[\left(\ell(z,\theta)-\ell(z,\theta^*)\right)^2\right]\leq B\left(\mathbb{E}\ell(z,\theta)-\mathbb{E}\ell(z,\theta^*)\right).
	\end{equation}
	
	With $G$-Lipschitz and PL inequality with parameter $\mu$, we have $B=2G^2/\mu$.
	
	So, via (\ref{equGBC}),
	\begin{equation}\label{B}
	\begin{aligned}
	\mathbb{E}\left[\left(\mathbb{E}'\ell(z_i,\hat{\theta}_n')-\ell(z_i,\theta^*)\right)^2\right]&\leq \frac{2G^2}{\mu}\left(\mathbb{E}\mathbb{E}'\ell(z_i,\hat{\theta}_n')-\mathbb{E}\ell(z_i,\theta^*)\right) \\
	&=\frac{2G^2}{\mu}\left(\mathbb{E}[R(\hat{\theta}_n')]-R(\theta^*)\right),
	\end{aligned}
	\end{equation}
	where the last equation holds because $\mathbb{E}\mathbb{E}'\ell(z_i,\hat{\theta}_n')=\mathbb{E}[R(\hat{\theta}_n')]$.
	
	Note that term $\mathbb{E}'\ell(z_i,\hat{\theta}_n')-\ell(z_i,\theta^*)$ can be decomposed as the following:
	\begin{equation*}
	\underbrace{\mathbb{E}'\ell(z_i,\hat{\theta}_n')-\ell(z_i,\theta^*)}_{X_i}=\underbrace{\mathbb{E}'\ell(z_i,\hat{\theta}_n')-\mathbb{E}'\ell(z_i,\theta_n')}_{X_i'}+\underbrace{\mathbb{E}'\ell(z_i,\theta_n')-\ell(z_i,\theta^*)}_{X_i''}.
	\end{equation*}
	
	Via triangle inequality,
	\begin{equation*}
	\left\Vert X_i\right\Vert_q\leq\left\Vert X_i'\right\Vert_q+\left\Vert X_i''\right\Vert_q.
	\end{equation*}
	
	Recalling the definition of $R_n(\hat{\theta}_n)-R_n(\theta_n^*)$, we have:
	\begin{equation}\label{Xi'}
	\left\Vert\frac{1}{n}\sum_{i=1}^{n}X_i'\right\Vert_q=R_n(\hat{\theta}_n)-R_n(\theta_n^*).
	\end{equation}
	
	Via Lemma \ref{l4}, since $\mathbb{E}[R(\theta_n')]-R(\theta^*)$ is exactly the expectation of each $X_i''$, we have for $q\geq2$,
	\begin{equation}\label{Xi''}
	\begin{aligned}
	\left\Vert\frac{1}{n}\sum_{i=1}^{n}\mathbb{E}'\left[\ell(z_i,\theta_n')\right]-\ell(z_i,\theta^*)-\mathbb{E}[R(\theta_n')]+R(\theta^*)\right\Vert_q&\lesssim\sqrt{\mathbb{E}\left[\left(\mathbb{E}'\ell(z_i,\hat{\theta}_n')-\ell(z_i,\theta^*)\right)^2\right]\frac{q}{n}} \\
	&\leq\sqrt{\frac{2G^2}{\mu}\left(\mathbb{E}[R(\theta_n)]-R(\theta^*)\right)\frac{q}{n}}+\frac{qM_\ell}{n},
	\end{aligned}
	\end{equation}
	where the last inequality holds because of (\ref{B}) and $\mathbb{E}[R(\theta_n)]=\mathbb{E}[R(\hat{\theta}_n')]$.
	
	Plugging (\ref{Xi'}) and (\ref{Xi''}) back into (\ref{EPRE'}), we obtain for each $q\geq2$ and some constant $C>0$,
	\begin{equation}\label{1n}
	\begin{aligned}
	&\left\Vert R(\hat{\theta}_n)-R(\theta^*)-\left(R_n(\hat{\theta}_n)-R_n(\theta_n^*)\right)\right\Vert_q \\
	&\leq C\left(Gq\log(n)\left(\sqrt{\frac{R_n(\hat{\theta}_n)-R_n(\theta_n^*)}{\mu}}+\frac{G}{\mu n}\right)+\left(R_n(\hat{\theta}_n)-R_n(\theta_n^*)\right)+\sqrt{\frac{2G^2}{\mu}\left(\mathbb{E}[R(\theta_n)]-R(\theta^*)\right)\frac{q}{n}}+\frac{qM_\ell}{n}\right) \\
	&\leq\varphi C\left(\mathbb{E}[R(\theta_n)]-R(\theta^*)\right)+C\left(Gq\log(n)\sqrt{\frac{R_n(\hat{\theta}_n)-R_n(\theta_n^*)}{\mu}}+\left(R_n(\hat{\theta}_n)-R_n(\theta_n^*)\right)\right) \\
	&\quad+C\left(\left(\frac{G^2\log(n)}{\mu}+\frac{2G^2}{\mu\varphi}+M_\ell\right)\frac{q}{n}\right),
	\end{aligned}
	\end{equation}
	where the last inequality holds because for $a,b,\varphi>0$, $\sqrt{ab}\leq\varphi a+b/\varphi$.
	
	Taking $q=2$, and via Cauchy-Schwarz inequality,
	\begin{equation*}
	\begin{aligned}
	&\mathbb{E}[R(\hat{\theta}_n)]-R(\theta^*)-\mathbb{E}\left[R_n(\hat{\theta}_n)-R_n(\theta_n^*)\right] \\
	&\leq\left\Vert R(\hat{\theta}_n)-R^*-\left(R_n(\hat{\theta}_n)-R_n(\theta_n^*)\right)\right\Vert_2 \\
	&\leq\varphi C\left(\mathbb{E}[R(\hat{\theta}_n)]-R(\theta^*)\right)+C\left(2G\log(n)\sqrt{\frac{R_n(\hat{\theta}_n)-R_n(\theta_n^*)}{\mu}}+\left(R_n(\hat{\theta}_n)-R_n(\theta_n^*)\right)\right) \\
	&\quad+\frac{2C}{n}\left(\frac{G^2\log(n)}{\mu}+\frac{2G^2}{\mu\varphi}+M_\ell\right).
	\end{aligned}
	\end{equation*}
	
	The inequality above can be rewritten as:
	\begin{equation*}
	\begin{aligned}
	\mathbb{E}[R(\hat{\theta}_n)]-R(\theta^*)&\leq\frac{1}{1-\varphi C}\mathbb{E}[R_n(\hat{\theta}_n)-R_n(\theta_n^*)]+\frac{C}{1-\varphi C}\left(R_n(\hat{\theta}_n)-R_n(\theta_n^*)\right)+\frac{2CG\log(n)}{1-\varphi C}\sqrt{\frac{R_n(\hat{\theta}_n)-R_n(\theta_n^*)}{\mu}} \\
	&\quad+\frac{2C}{(1-\varphi C)n}\left(\frac{G^2\log(n)}{\mu}+\frac{2G^2}{\mu\varphi}+M_\ell\right).
	\end{aligned}
	\end{equation*}
	
	Taking this back to (\ref{1n}), we have:
	\begin{equation}\label{finalhat}
	\begin{aligned}
	R(\hat{\theta}_n)-R(\theta^*)&\leq c_1\left(R_n(\hat{\theta}_n)-R_n(\theta_n^*)\right)+c_2\log(n)\sqrt{R_n(\hat{\theta}_n)-R_n(\theta_n^*)}+c_3\frac{\log(n)}{n},
	\end{aligned}
	\end{equation}
	for some constants $c_1,c_2$ and $c_3$, where we combine $R_n(\hat{\theta}_n)-R_n(\theta_n^*)$ and its expectation together.

	Then via Lemma \ref{l6}, with probability at least $1-\zeta$, we have:
	\begin{equation*}
	\begin{aligned}
	R(\hat{\theta}_n)-R(\theta^*)&\leq c_1\frac{G^2p\log(n)\log(1/\delta)}{n^2\epsilon^2}\left(1+\left(\frac{8\log(T/\zeta)}{p}\right)^{1/4}\right)^2 \\
	&\quad+c_2\frac{G\log^{1.5}(n)\sqrt{p\log(1/\delta)}}{n\epsilon}\left(1+\left(\frac{8\log(T/\zeta)}{p}\right)^{1/4}\right)+c_3\frac{\log(n)}{n}.
	\end{aligned}
	\end{equation*}
	for some constants $c_1,c_2,c_3>0$.
	
	The result follows.
\end{proof}

\subsection{A.4. Proof of Theorem \ref{t3}}

\begin{theorem}
	If the loss function is $\alpha$-H{\"o}lder smooth (Assumption \ref{a3}) with parameter $H$, and satisfies the PL condition with parameter $2\mu$ (Assumption \ref{a5}), the loss function and the parameter space are bounded, i.e. $0\leq\ell(\cdot,\cdot)\leq M_\ell$, $\|\mathcal{C}\|_2\leq M_\mathcal{C}$.
	Taking $\sigma$ given by Lemma \ref{l1}, $T=\mathcal{O}\left(n^{\frac{2}{1+2\alpha}}\right)$, and $\eta_t=\frac{2}{\mu(t+\kappa)}$, where $\kappa\geq\frac{2H^{1/\alpha}}{\mu}$, if $\zeta\in(\exp(-p/8),1)$, then with probability at least $1-\zeta$:
	\begin{equation*}
	\begin{aligned}
	R(\hat{\theta}_n)-R(\theta^*)&\leq c_1\frac{G'^2\sqrt{p\log(1/\delta)}}{n^{\frac{2\alpha}{1+2\alpha}}\epsilon}\left(1+\left(\frac{8\log(T/\zeta)}{p}\right)^{1/4}\right) \\
	&\quad+c_2\frac{G'\log(n)\sqrt[4]{p\log(1/\delta)}}{n^{\frac{\alpha}{1+2\alpha}}\epsilon^{1/2}}\left(1+\left(\frac{8\log(T/\zeta)}{p}\right)^{1/4}\right)^{1/2}+c_3\frac{\log(n)}{n},
	\end{aligned}
	\end{equation*}
	for some constants $c_1,c_2$ and $c_3$, where $G'=\max\{2HM_\mathcal{C},H\}$.
\end{theorem}

\begin{proof}
	
	Like inequality (\ref{finalhat}) in the proof of Theorem \ref{t2} (Appendix A.3), we have:
	\begin{equation}\label{EPRHS}
	\begin{aligned}
	R(\hat{\theta}_n)-R(\theta^*)&\leq c_1\left(R_n(\hat{\theta}_n)-R_n(\theta_n^*)\right)+c_2\log(n)\sqrt{R_n(\hat{\theta}_n)-R_n(\theta_n^*)}+c_3\frac{\log(n)}{n},
	\end{aligned}
	\end{equation}
	for some constants $c_1,c_2$ and $c_3$.
	
	The differences between inequalities (\ref{finalhat}) and (\ref{EPRHS}) are in constants, for example, Lipschitz constant $G$ discussed in Appendix A.3 comes to $G'=\max\{2HM_\mathcal{C},H\}$ here, as discussed before; and in Appendix A.3, the PL condition is with parameter $\mu$, rather than $2\mu$ here.
	
	Combining the result obtained by Lemma \ref{l7}, taking $T=\mathcal{O}\left(n^{\frac{2}{1+2\alpha}}\right)$ and $\eta_t=\frac{2}{\mu(t+\kappa)}$ with $\kappa\geq\frac{2H^{1/\alpha}}{\mu}$, then with probability at least $1-\zeta$,
	\begin{equation*}
	\begin{aligned}
	R(\hat{\theta}_n)-R(\theta^*)&\leq c_1\frac{G'^2\sqrt{p\log(1/\delta)}}{n^{\frac{2\alpha}{1+2\alpha}}\epsilon}\left(1+\left(\frac{8\log(T/\zeta)}{p}\right)^{1/4}\right) \\
	&\quad+c_2\frac{G'\log(n)\sqrt[4]{p\log(1/\delta)}}{n^{\frac{\alpha}{1+2\alpha}}\epsilon^{1/2}}\left(1+\left(\frac{8\log(T/\zeta)}{p}\right)^{1/4}\right)^{1/2}+c_3\frac{\log(n)}{n},
	\end{aligned}
	\end{equation*}
	for some constants $c_1,c_2$ and $c_3$, where $G'=\max\{2HM_\mathcal{C},H\}$.
	
	The result holds.
\end{proof}

\subsection{A.5. Proof of Theorem \ref{t4}}

\begin{theorem}
	If Assumptions \ref{a3}, \ref{a5} hold, the loss function and the parameter space are bounded, i.e. $0\leq\ell(\cdot,\cdot)\leq M_\ell$, $\|\mathcal{C}\|_2\leq M_\mathcal{C}$.
	Taking $\sigma$ given by Lemma \ref{l1}, $T=\mathcal{O}\left(\log(n)\right)$, and $\eta_1=\cdots=\eta_T=\eta$, where $\eta=\left(\frac{1}{H}\right)^{1/\alpha}$, if $\zeta\in(\exp(-p/8),1)$, then with probability at least $1-\zeta$,
	\begin{equation*}
	\begin{aligned}
	R(\hat{\theta}_n)-R(\theta^*)
	&\leq c_1\frac{G'^2\log(n)p\log(1/\delta)}{n^2\epsilon^2}\left(1+\left(\frac{8\log(T/\zeta)}{p}\right)^{1/4}\right)^2 \\
	&\quad+c_2\frac{G'\log^{1.5}(n)\sqrt{p\log(1/\delta)}}{n\epsilon}\left(1+\left(\frac{8\log(T/\zeta)}{p}\right)^{1/4}\right)+c_3\frac{\log(n)}{n},
	\end{aligned}
	\end{equation*}
	for some constants $c_1,c_2,c_3>0$, where $G'=\max\{2HM_\mathcal{C},H\}$.
\end{theorem}

\begin{proof}
	The proof is similar to Theorems \ref{t2} and \ref{t3}, we first analyze the optimization error $R_n(\hat{\theta}_n)-R_n(\theta_n^*)$.
	
	For algorithm \ref{alg1}, with normalization, if taking $\eta_t=H^{-1/\alpha}$ we have:
	\begin{equation*}
	\begin{aligned}
	R_n(\hat{\theta}_{t+1})-R_n(\hat{\theta}_t)&\overset{(\alpha)}{\leq}\langle\nabla_\theta R_n(\hat{\theta}_t),\hat{\theta}_{t+1}-\hat{\theta}_t\rangle+\frac{H}{2}\left\Vert\hat{\theta}_{t+1}-\hat{\theta}_t\right\Vert_2^{\alpha+1} \\
	&=-\eta_t\langle\nabla_\theta R_n(\hat{\theta}_t),\nabla_\theta R_n(\hat{\theta}_t)+b_t\rangle+\frac{H\eta_t^{\alpha+1}}{2}\left(\left\Vert\nabla_\theta R_n(\hat{\theta}_t)+b_t\right\Vert_2\right)^{\alpha+1} \\
	&\leq-\eta_t\left\Vert\nabla_\theta R_n(\hat{\theta}_t)\right\Vert_2^2+\frac{H\eta_t^{\alpha+1}}{2}\left(\left\Vert\nabla_\theta R_n(\hat{\theta}_t)\right\Vert_2+\left\Vert b\right\Vert_2\right)^{2}+\left(H\eta_t^{\alpha+1}-\eta_t\right)\langle\nabla_\theta R_n(\hat{\theta}_t),b_t\rangle \\
	&\leq\frac{-\eta_t}{2}\left\Vert\nabla_\theta R_n(\hat{\theta}_t)\right\Vert_2^2+\frac{H\eta_t^{\alpha+1}}{2}\|b\|_2^2 \\
	&\overset{(PL)}{\leq}-\mu\eta_t\left(R_n(\hat{\theta}_t)-R_n(\theta_n^*)\right)+\frac{\eta_t}{2}\|b\|_2^2,
	\end{aligned}
	\end{equation*}
	where the second inequality holds because by normalization, and the third inequality holds because $\eta_t=\left(\frac{1}{H}\right)^{1/\alpha}$.
	
	Summing $R_n(\theta_t)-R_n(\theta_n^*)$ to both sides, we have:
	\begin{equation*}
	R_n(\hat{\theta}_{t+1})-R_n(\theta_n^*)\leq\left(1-\mu\eta_t\right)\left(R_n(\hat{\theta}_t)-R_n(\theta_n^*)\right)+\frac{\eta_t}{2}\|b\|_2^2.
	\end{equation*}
	
	Summing over $T$ iterations,
	\begin{equation*}
	\begin{aligned}
	R_n(\hat{\theta}_T)-R_n(\theta_n^*)&\leq\left(1-\mu\eta\right)^T\left(R_n(\hat{\theta}_0)-R_n(\theta_n^*)\right)+\frac{\eta}{2}\sum_{t=0}^{T-1}\left(1-\mu\eta\right)^t\|b_t\|_2^2,
	\end{aligned}
	\end{equation*}
	where $\eta=\left(\frac{1}{H}\right)^{1/\alpha}$.
	
	With Lemma \ref{l5}, with probability at least $1-\xi$,
	\begin{equation*}
	\begin{aligned}
	R_n(\hat{\theta}_n)-R_n(\theta_n^*)&\leq\left(1-\mu\eta\right)^TM_\ell+\frac{\sigma^2p}{2\mu}\left(1+\left(\frac{8\log(T/\zeta)}{p}\right)^{1/4}\right)^2.
	\end{aligned}
	\end{equation*}
	
	Noting that by definition, $\mu\eta\leq1$, so if taking $T=\mathcal{O}\left(\log(n)\right)$, with probability at least $1-\zeta$, we have:
	\begin{equation*}
	R_n(\hat{\theta}_n)-R_n(\theta_n^*)\leq c\frac{G'^2\log(n)p\log(1/\delta)}{n^2\epsilon^2}\left(1+\left(\frac{8\log(T/\zeta)}{p}\right)^{1/4}\right)^2.
	\end{equation*}
	
	Then, like in (\ref{EPRHS}), we have:
	\begin{equation*}
	\begin{aligned}
	R(\hat{\theta}_n)-R(\theta^*)&\leq c_1\left(R_n(\hat{\theta}_n)-R_n(\theta_n^*)\right)+c_2\log(n)\sqrt{R_n(\hat{\theta}_n)-R_n(\theta_n^*)}+c_3\frac{\log(n)}{n} \\
	&\leq c_1\frac{G'^2\log(n)p\log(1/\delta)}{n^2\epsilon^2}\left(1+\left(\frac{8\log(T/\zeta)}{p}\right)^{1/4}\right)^2 \\
	&\quad+c_2\frac{G'\log^{1.5}(n)\sqrt{p\log(1/\delta)}}{n\epsilon}\left(1+\left(\frac{8\log(T/\zeta)}{p}\right)^{1/4}\right)+c_3\frac{\log(n)}{n},
	\end{aligned}
	\end{equation*}
	for some constants $c_1,c_2,c_3>0$.
	
	The result follows.
	
\end{proof}

\section{B. More Experimental Results}

\subsection{B.1. Accuracies on More Datasets}

In this section, we show the experimental results on datasets Breast Cancer, Credit Card Fraud, and Bank.
Details are shown in Figure \ref{fig4}.

The results are similar to which given by Figure \ref{fig3} in Section 5: although there are some fluctuations over some datasets (such as Bank), the performance of our proposed m-NGP method is similar to or better than traditional method on most datasets.

\begin{figure}[ht]
	\vskip 0.2in
	\begin{center}
		\centering{
			\subfigure[Breast Cancer]{\includegraphics[width=0.3\textwidth]{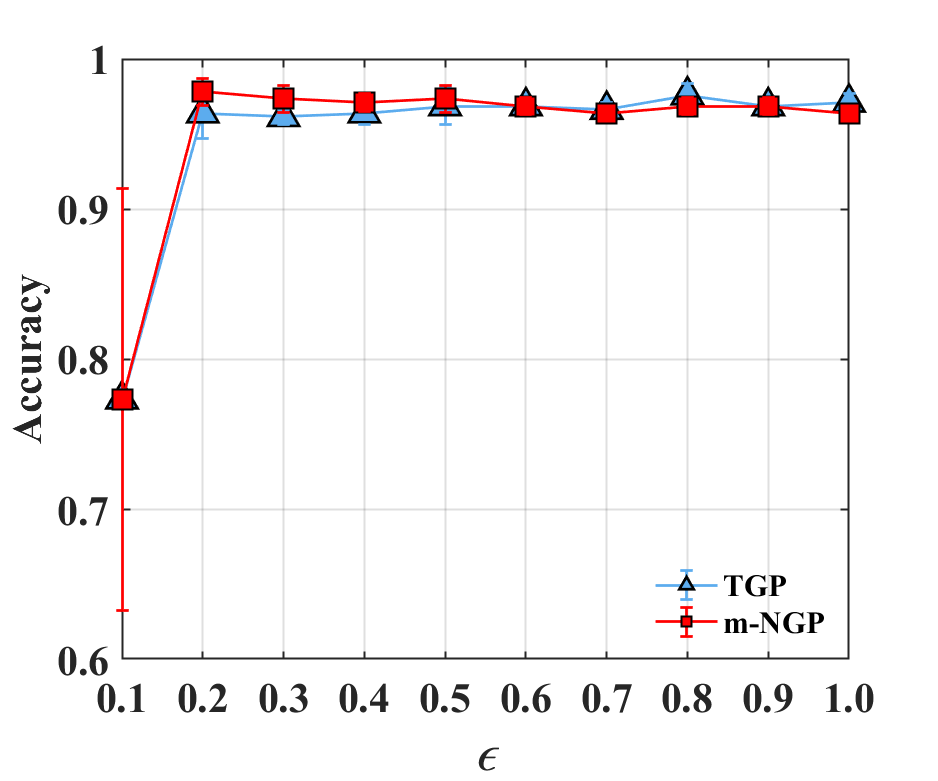}}
			\subfigure[Credit Card Fraud]{\includegraphics[width=0.3\textwidth]{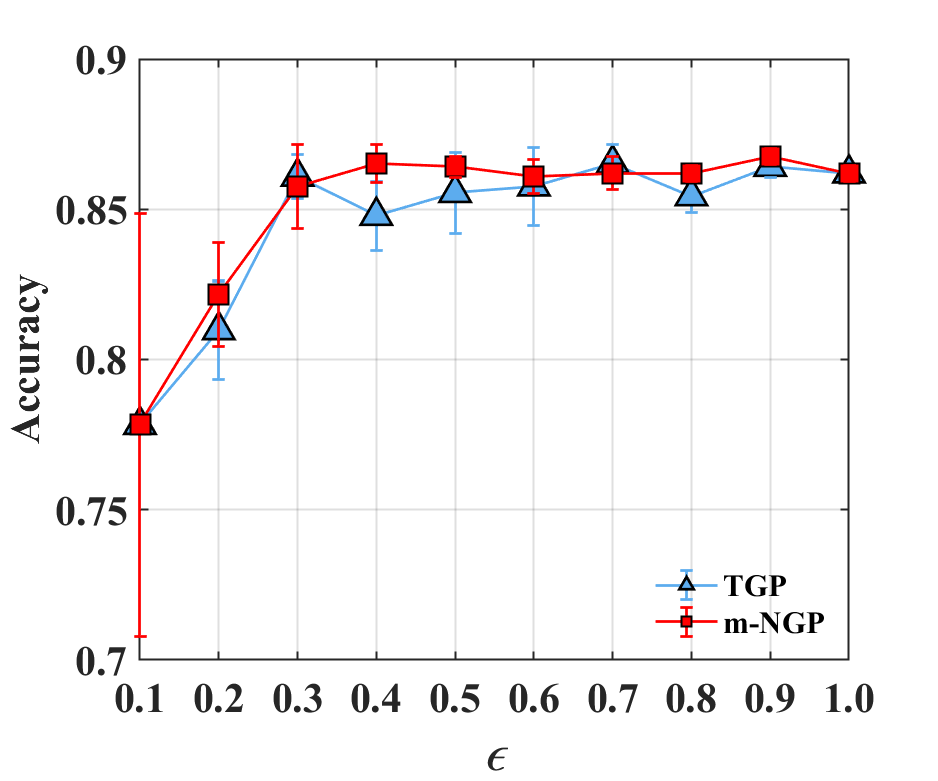}}
			\subfigure[Bank]{\includegraphics[width=0.3\textwidth]{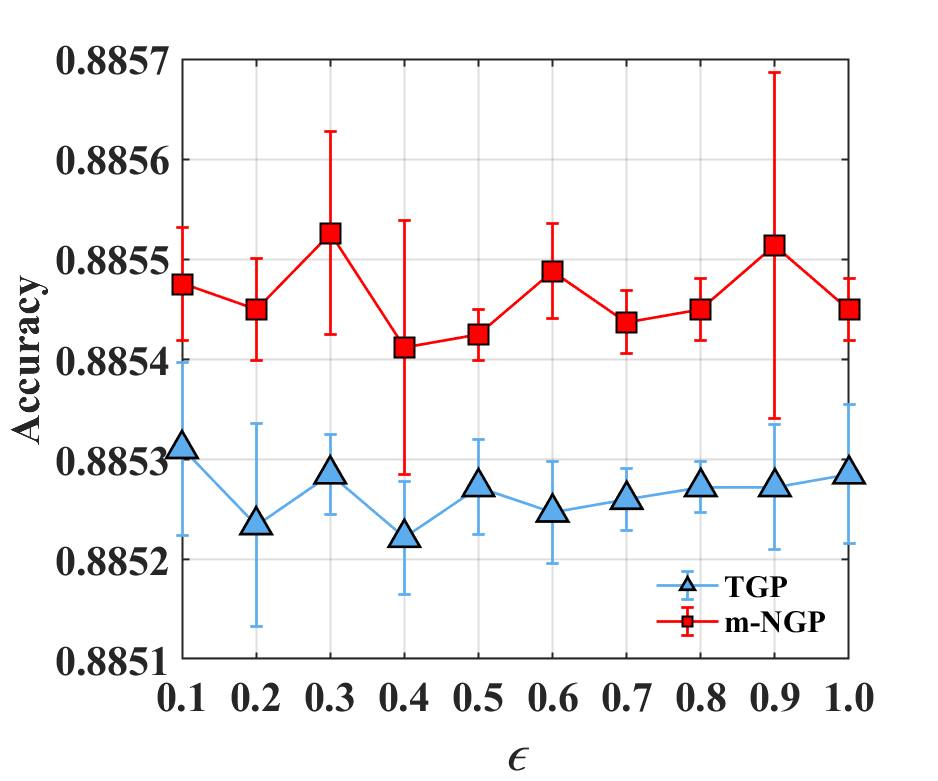}}
		}
	\end{center}
	\caption{Comparisons between Traditional Gradient Perturbation (TGP) method and max$\bf\{1,g\}$-Normalized Gradient Perturbation (m-NGP) method.}
	\label{fig4}
\end{figure}

\subsection{B.2. Convergence Rate and Normalization}

In this section, we perform experiments to demonstrate the effects on the convergence rate caused by normalization when applying m-NGP.
The privacy budget $\epsilon$ is set 0.5.
Detailed results are shown in Figure \ref{fig1}.

In Figure \ref{fig1}, the lines with dark color and light color correspond to m-NGP and TGP, respectively, and the shadow area represents the maximum and minimum loss over mutiple experiments, reflecting the variance.
And the horizontal axis is iterations and the ordinate is the loss.
The experimental results show that over most datasets, m-NGP (normalization) achieves faster convergence rate, comparing with TGP, which is in line with the theoretical analysis.

\begin{figure}[ht]
	\vskip 0.2in
	\begin{center}
		\centering{
			\subfigure[Breast Cancer]{\includegraphics[width=0.3\textwidth]{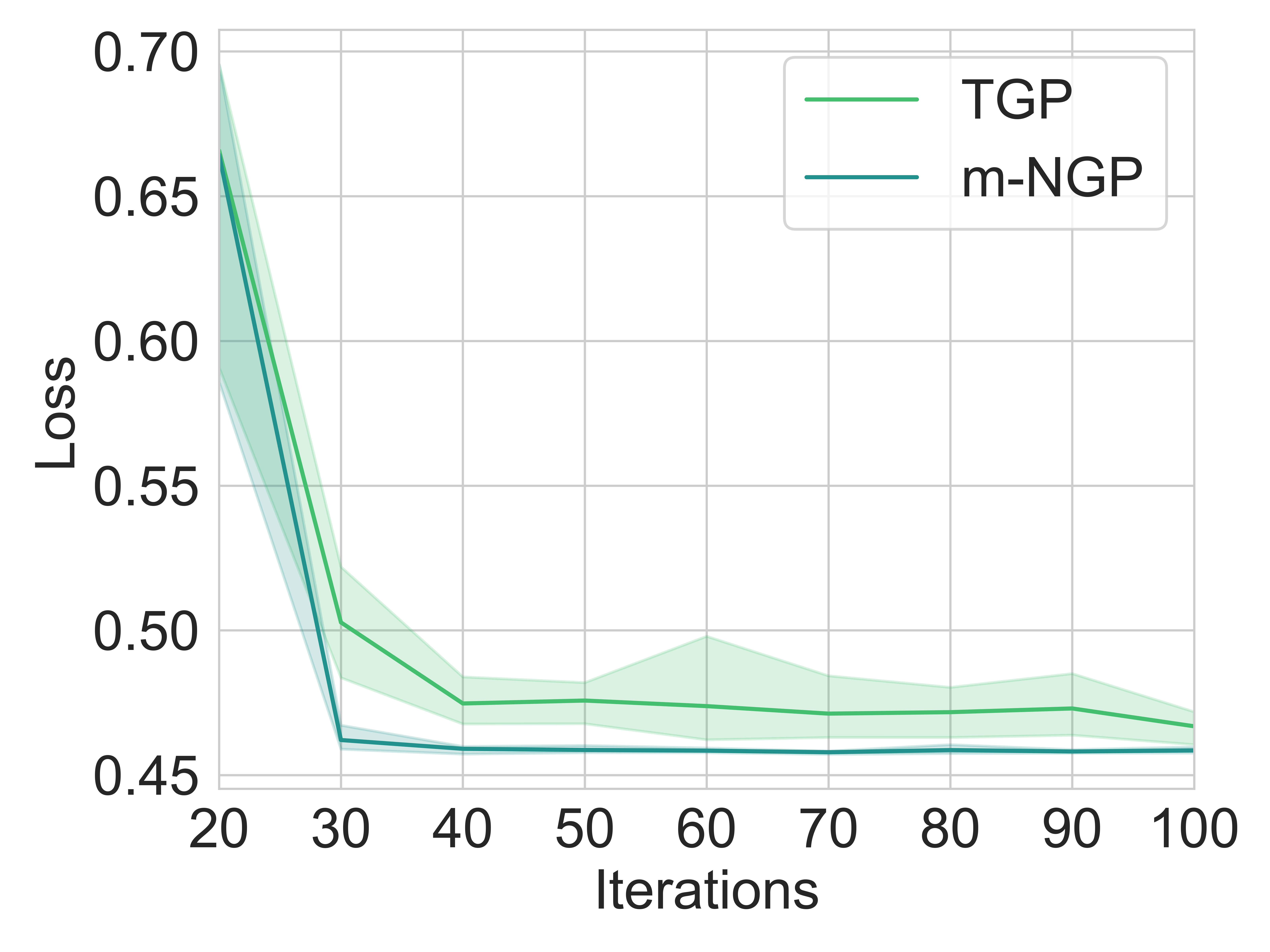}}
			\subfigure[Credit Card Fraud]{\includegraphics[width=0.3\textwidth]{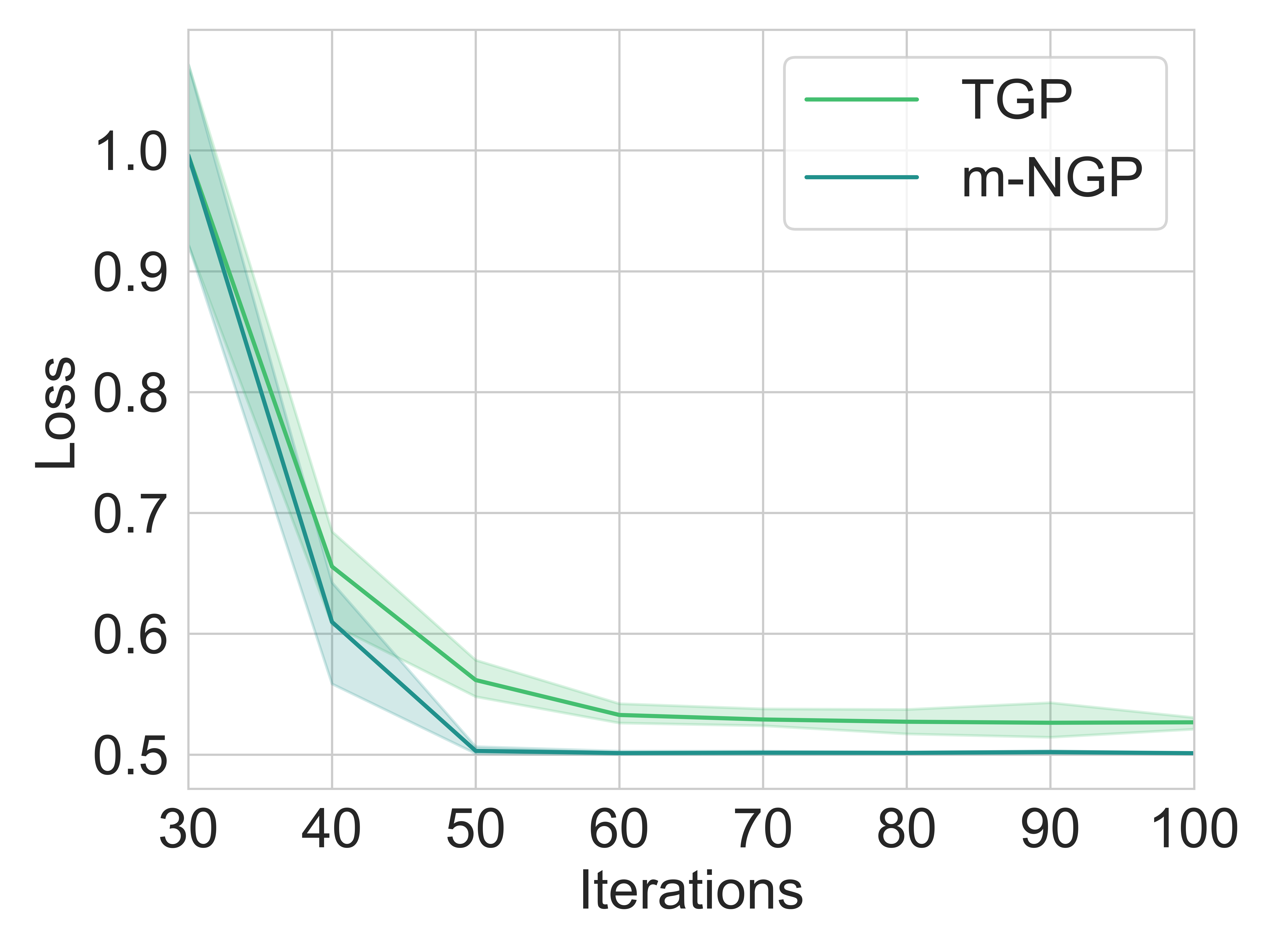}}
			\subfigure[Bank]{\includegraphics[width=0.3\textwidth]{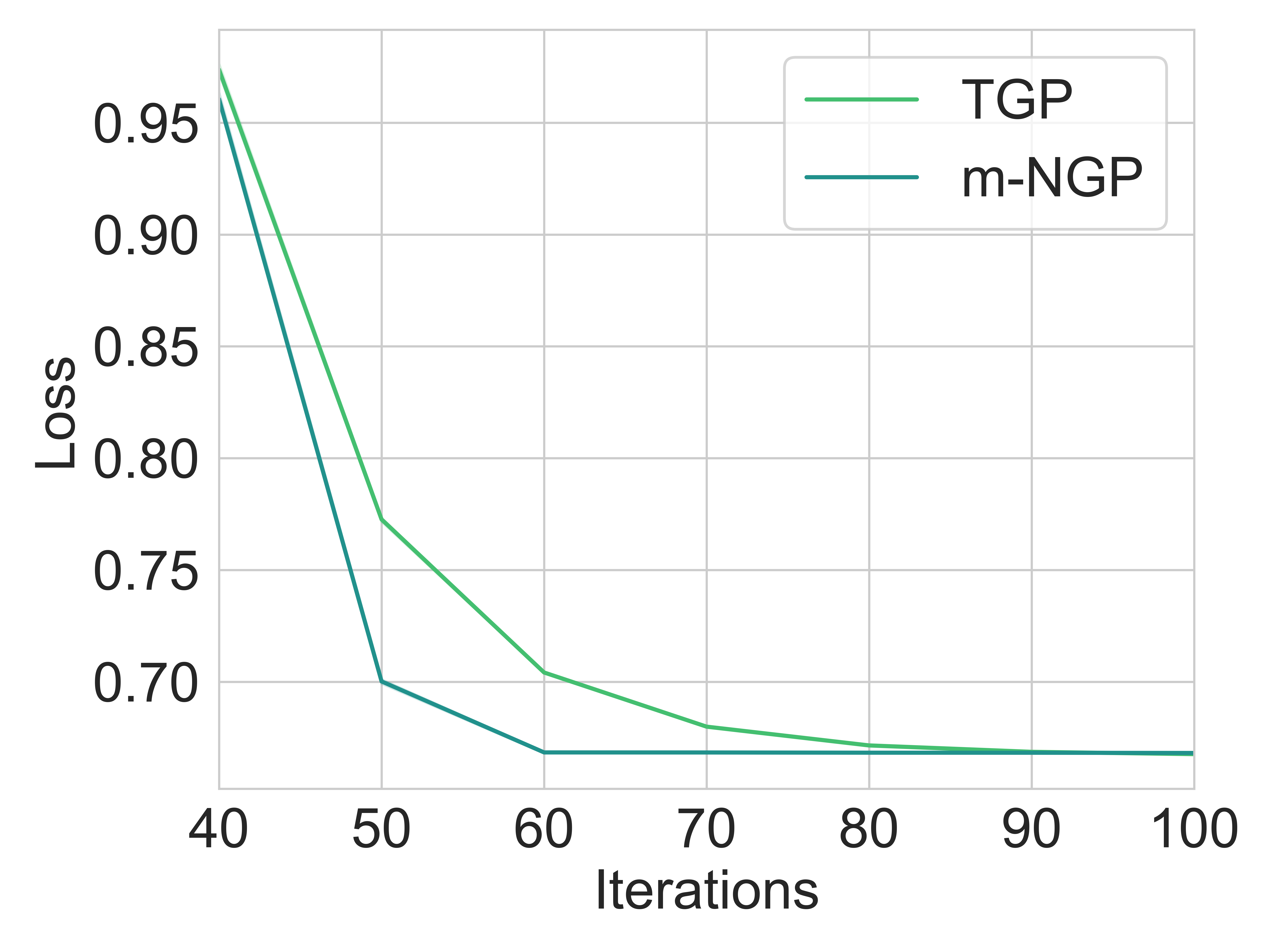}}
		}
	\end{center}
	\caption{Convergence Rates of TGP and m-NGP.}
	\label{fig1}
\end{figure}

\subsection{B.3. Accuracy and Dimension $p$}

In this section, we perform experiments to demonstrate the effects on the accuracy brought by the dimensions of data instances.
The experiments are performed on datasets Credit Card Fraud, Bank, and Adult, whose dimensions are 29, 48, and 104, respectively.
And the privacy budget $\epsilon$ is set 0.5.
The results are shown in Figure \ref{fig2}.

For abscissa, the first dimensions of parts (a), (b), and (c) are set $p=29,48,104$, they are original features given by the datasets.
And the dimensions more than them are all set $0$, to evaluate the effects brought by the magnitude of $p$, without introducing new information.

\begin{figure}[ht]
	\vskip 0.2in
	\begin{center}
		\centering{
			\subfigure[Credit Card Fraud]{\includegraphics[width=0.3\textwidth]{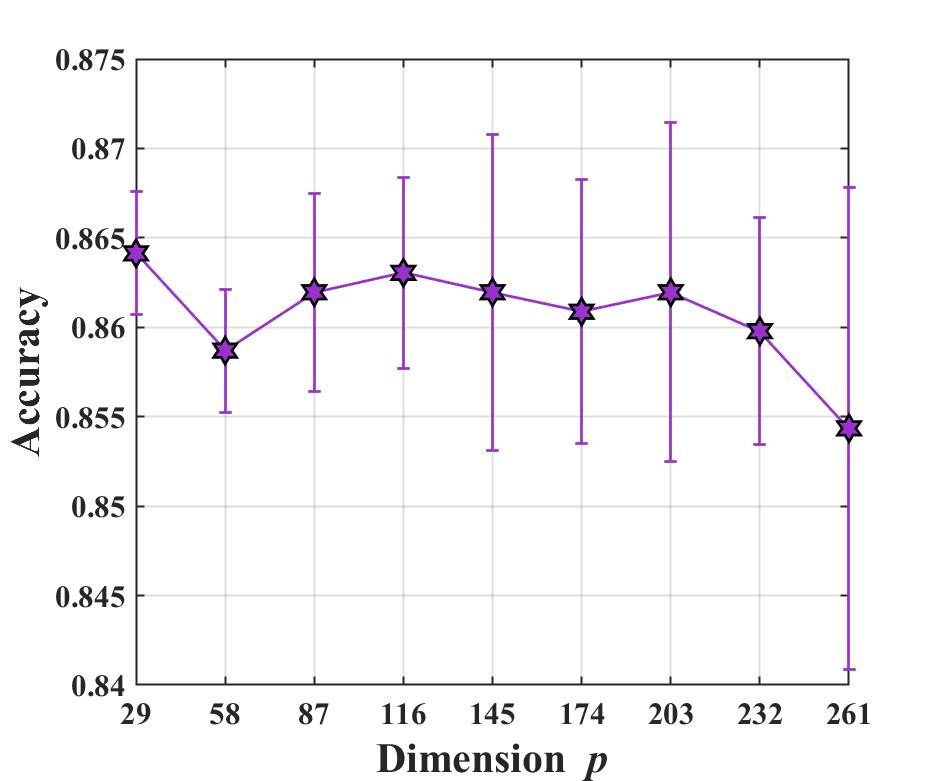}}
			\subfigure[Bank]{\includegraphics[width=0.3\textwidth]{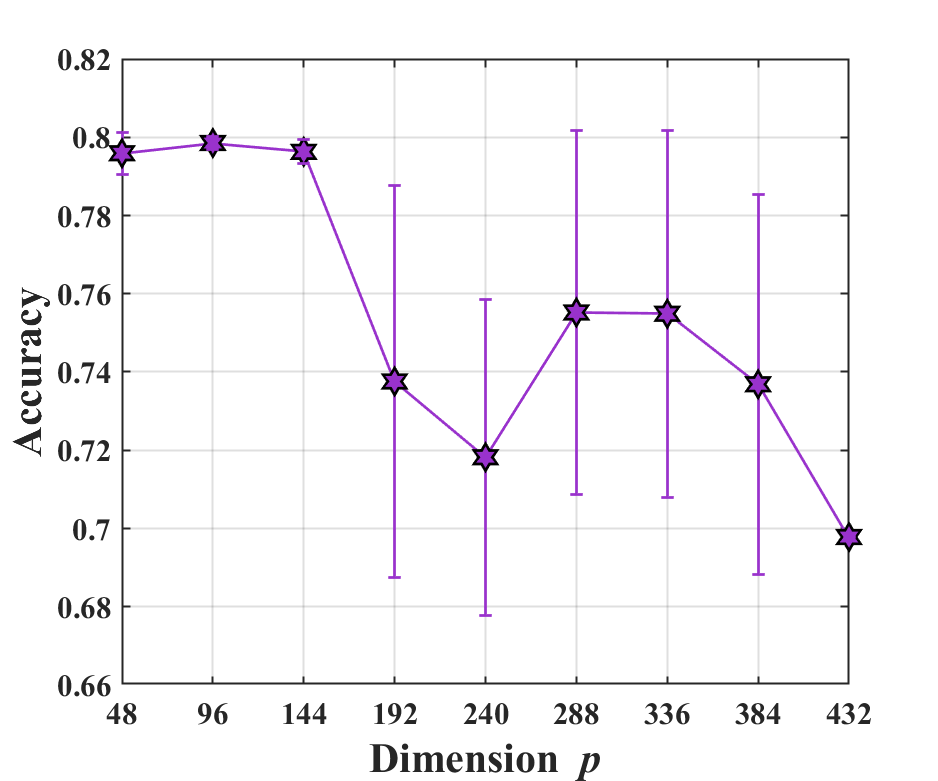}}
			\subfigure[Adult]{\includegraphics[width=0.3\textwidth]{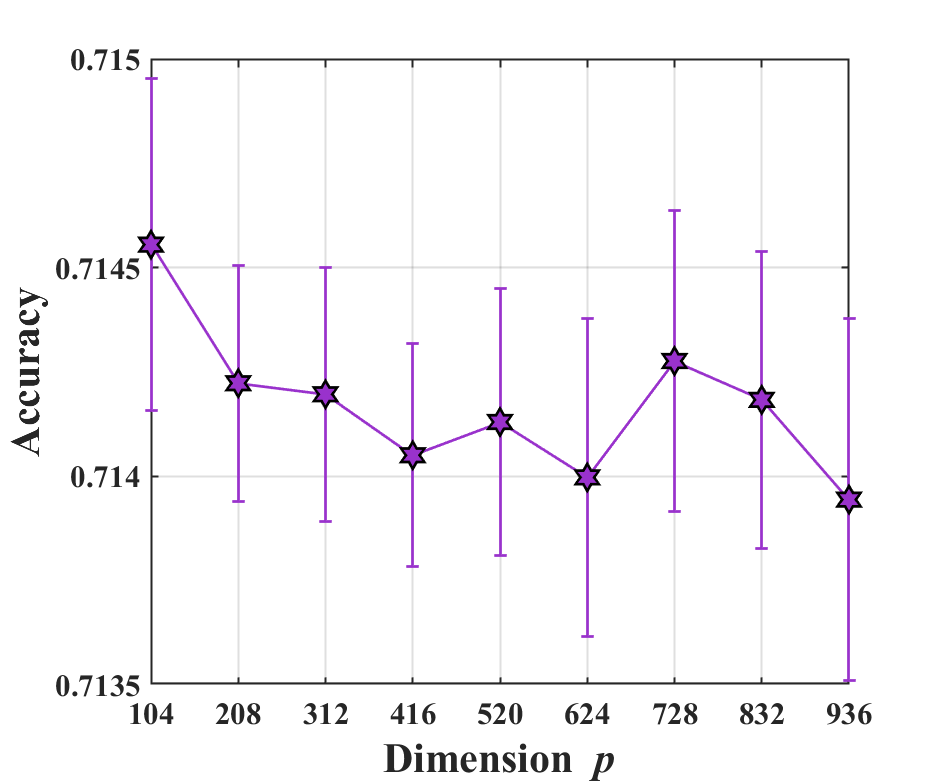}}
		}
	\end{center}
	\caption{Effects of dimension $p$ on m-NGP.}
	\label{fig2}
\end{figure}

Experimental results show that although there may exist some fluctuations caused by the injected random noise, the accuracy decreases with the increasing of $p$ overall, which is in line with the theoretical analysis given in Section 4.

\end{appendix}

\end{document}